\RequirePackage{xcolor}
\documentclass[journal,twoside,web]{IEEEtran}

\usepackage{adjustbox}
\usepackage{cite}
\usepackage{amsmath,amssymb,amsfonts}
\usepackage{graphicx}
\usepackage{algorithm}
\usepackage[noend]{algpseudocode}
\usepackage{hyperref}
\hypersetup{colorlinks = true, linkcolor = red, citecolor = blue, urlcolor = blue}
\usepackage{textcomp}
\usepackage{booktabs}
\usepackage{tabularx}
\usepackage{colortbl}
\usepackage{multirow}
\newcolumntype{Y}{>{\centering\arraybackslash}X}

\usepackage{amsthm}
\usepackage[all]{hypcap}                     %
\usepackage[capitalise,noabbrev]{cleveref}   %
\usepackage{autonum}                         %
\crefformat{equation}{(#2#1#3)}              %

\usepackage{yhmath}       %
\usepackage{comment}      %
\usepackage{mathtools}    %
\usepackage{mleftright}   %
\usepackage{stmaryrd}     %
\usepackage{subcaption}
\usepackage{bbm}
\usepackage{balance}
\usepackage{pgfplots}   %
\usepackage{tikz}
\usetikzlibrary{automata,calc}

\pgfplotsset{compat=1.18}   %

\usepackage{subcaption}
\usepackage{enumitem}
\setlist[enumerate]{leftmargin=*}
\setlist[itemize]{leftmargin=*}

\allowdisplaybreaks   %

\newtheorem{theorem}{Theorem}

\newtheorem{lemma}[theorem]{Lemma}
\newtheorem{proposition}[theorem]{Proposition}

\AddToHook{env/corollary/begin}{\crefalias{theorem}{corollary}}
\AddToHook{env/definition/begin}{\crefalias{theorem}{definition}}
\AddToHook{env/lemma/begin}{\crefalias{theorem}{lemma}}
\AddToHook{env/proposition/begin}{\crefalias{theorem}{proposition}}

\newcommand{\equad}{\mathrel{\phantom{=}}}

\newcommand{\abs}[1]{\left| #1 \right|}

\newcommand{\bkt}[1]{\left[ #1 \right]}
\newcommand{\bktz}[1]{\left\llbracket #1 \right\rrbracket}
\newcommand{\brc}[1]{\left\{ #1 \right\}}
\newcommand{\ceil}[1]{\left\lceil #1 \right\rceil}

\newcommand{\norm}[1]{\left\lVert #1 \right\rVert}
\newcommand{\prn}[1]{\left( #1 \right)}

\newcommand{\mbrc}[1]{\mleft\{ #1 \mright\}}
\newcommand{\mprn}[1]{\mleft( #1 \mright)}

\newcommand{\ibktz}[1]{\llbracket #1 \rrbracket}
\newcommand{\ibrc}[1]{\{ #1 \}}
\newcommand{\inorm}[1]{\lVert #1 \rVert}
\newcommand{\iprn}[1]{( #1 )}
\newcommand{\bigprn}[1]{\bigl( #1 \bigr)}

\newcommand{\biggprn}[1]{\biggl( #1 \biggr)}

\newcommand{\Iv}[1]{\mathbbm{1} \mbrc{#1}}
\newcommand{\iIv}[1]{\mathbbm{1} \ibrc{#1}}

\newcommand{\stackclap}[2]{\stackrel{\mathclap{#1}}{#2}}   %

\renewcommand{\cdots}{\mkern-3mu\mathinner{\cdotp\mkern-3mu\cdotp\mkern-3mu\cdotp}\mkern-3mu}
\renewcommand{\dots}{\mathinner{\ldotp\mkern-3mu\ldotp\mkern-3mu\ldotp}}

\newcommand{\fsmall}[2]{f({#1}; {#2})} %
\newcommand{\p}{p} %
\renewcommand{\d}{d} %
\newcommand{\nonnegativeintegers}{\mathbb{Z}_{\geq 0}}
\newcommand{\positiveintegers}{\mathbb{Z}_{> 0}}
\renewcommand{\vec}[1]{{\boldsymbol{\mathbf{{#1}}}}} %
\newcommand{\define}{\triangleq}

\newcommand{\ep}{\varepsilon}
\newcommand{\grad}{\vec{\nabla}}
\newcommand{\oneton}[1]{\bkt{#1}} %
\newcommand{\uglyfrac}[2]{{{#1}/{#2}}} %
\newcommand{\uglyfracprn}[2]{{{#1}/{(#2)}}} %
\newcommand{\holderorder}{\ell} %
\newcommand{\holderconst}{{L_h}}
\newcommand{\looseholderconst}{L_h} %
\newcommand{\powerset}[1]{{2^{#1}}}
\newcommand{\deriv}[2]{{#1^{\prn{#2}}}} %
\newcommand{\lipschitzconst}{{L_l}}
\newcommand{\tsp}{^{\mkern-1.5mu\mathsf{T}}}
\newcommand{\patches}{\mathcal{P}}

\begin{document}

\bstctlcite{IEEEexample:BSTcontrol} 

\title{On the Oracle Complexity of Interpolation-Based Gradient Descent}
\author{Dongmin~Lee, William~Lu, and Anuran~Makur
    \thanks{The author ordering is alphabetical. This work was supported in part by the National Science Foundation (NSF) CAREER Award under Grant CCF-2337808.}
    \thanks{Dongmin Lee is with the Department of Computer Science, Purdue University, West Lafayette, IN 47907, USA (e-mail: lee4818@purdue.edu).}
    \thanks{William Lu is with the Department of Computer Science, Purdue University, West Lafayette, IN 47907, USA (e-mail: lu909@purdue.edu).}
    \thanks{Anuran Makur is with the Department of Computer Science and the Elmore Family School of Electrical and Computer Engineering, Purdue University, West Lafayette, IN 47907, USA (e-mail: amakur@purdue.edu).}
}

\maketitle

\begin{abstract}   %
Recent work on first-order optimizers for empirical risk minimization (ERM) has suggested that smoothness of ERM loss functions in the training data, rather than in the optimization parameters, can be leveraged to improve the oracle complexity of gradient descent (GD) methods. In this paper, we propose an inexact gradient method, \emph{piecewise polynomial interpolation-based gradient descent} (PPI-GD), which approximates the full gradient in each iteration by querying the first-order oracle at equidistant points in the data domain to construct polynomial interpolants of the resulting gradient samples over appropriately sized patches of the data domain. We analyze the oracle complexity of PPI-GD for strongly convex and non-convex loss functions when the data space dimension is bounded by a polylogarithmic function of the number of training samples, and find it to outperform several GD variants in key regimes when the loss function is sufficiently smooth. Furthermore, our analysis extends several techniques from the error analysis of bicubic spline interpolants to the setting of $d$-variate tensor product polynomial interpolants which may be of independent interest in interpolation analysis.
\end{abstract}


\section{Introduction}

First-order optimization algorithms such as gradient descent (GD) and its variants are commonly used within the paradigm of empirical risk minimization (ERM) to optimize loss functions over a finite training dataset. Theoretical work on the performance of such algorithms has centered around the notion of \emph{oracle complexity} introduced by \cite{NemirovskijYudin1983}, defined as the number of gradient queries necessary to guarantee convergence (e.g., to the optimal objective value) up to some absolute error $\ep$. Classical results in this area include upper bounds for GD under assumptions like convexity and strong convexity \cite{Nesterov2004}, upper bounds for stochastic gradient descent (SGD) with constant and decreasing step sizes \cite{Schmidt2014,LacosteJulienSchmidt2012}, and information-theoretic lower bounds \cite{Agarwaletal2009}, among others.

Building upon work on inexact gradient descent methods (e.g., \cite{dAspremont2008,FriedlanderSchmidt2012,DevolderGlineurNesterov2013,DevolderGlineurNesterov2014}), a recent line of research initiated by \cite{JadbabaieMakur2024} aims to improve upon the oracle complexity of classical gradient-based optimizers by exploiting the smoothness of the ERM loss function with respect to the data, a characteristic of many machine learning objectives which was hitherto neglected in previous literature on the subject. Specifically, the authors consider the ERM formulation (cf. \cite{Vapnik1991})
$
    F(\boldsymbol{\theta}) %
    = \frac{1}{n} \sum_{i=1}^n f(\mathbf{x}_i ; \boldsymbol{\theta})
$, %
where $n$ is the number of training samples and the loss function $f$ belongs to an $(\holderorder, \holderconst)$-H\"{o}lder class with respect to the training instance $\mathbf{x}_i$. The authors propose a method dubbed \emph{local polynomial interpolation-based gradient descent} (LPI-GD), which evaluates $\boldsymbol{\nabla}_{\boldsymbol{\theta}} f$ at a grid of virtual data points and performs local polynomial regression to estimate the full training gradient $\boldsymbol{\nabla}_{\boldsymbol{\theta}} F$ on each iteration without incurring $n$ oracle calls. For strongly convex problems with $d = O(\log \log (n))$ data dimensions and $p$ parameter dimensions, the authors establish the oracle complexity of LPI-GD as $\tilde{O}((p / \ep)^{d / (2 \ell)})$, beating the $O(n \log(p / \ep))$ complexity of GD and the $O(p / \ep)$ complexity of SGD when $\ell$ is sufficiently large, $p = O(\mathsf{poly}(n))$, and $\ep = \Theta(\mathsf{poly}(1/n))$. 
However, the assumption that $d = O(\log \log (n))$ limits the theoretical relevance of LPI-GD to low-dimensional datasets, and it is conjectured that this dependence can be improved to $d$ scaling polylogarithmically in $n$.

In this paper, motivated by this conjecture and intrinsic theoretical interest in the application of numerical analysis techniques within the oracle complexity domain, we introduce the \emph{piecewise polynomial interpolation-based gradient descent} (PPI-GD) method. PPI-GD achieves the aforementioned $\tilde{O}((p / \ep)^{d / (2 \ell)})$ oracle complexity for strongly convex ERM problems in the substantially broader regime of data dimensionality $d = O(\log^{0.49}(n))$, hence resolving the conjectured scaling regime of $d$. In addition, we also analyze the oracle complexity of PPI-GD for non-convex loss functions and show that it outperforms GD and SGD in a similar regime. We establish the oracle complexity of our method by generalizing proof techniques initially considered in the context of bicubic splines to the setting of multivariate polynomial interpolants. Next, we enumerate the principal contributions of our work and summarize the relevant literature on oracle complexity and first-order optimizers.

\subsection{Main Contributions}

Our work makes the following contributions:

\begin{enumerate}
    \item We present a new first-order optimization algorithm, PPI-GD (\cref{alg:ppigd}), to solve the ERM optimization problem when the loss function satisfies standard Lipschitz continuity and H\"{o}lder smoothness assumptions in the parameters and data instance, respectively. Unlike LPI-GD, which uses local polynomial regression across the entire data space to compute approximate gradients (which can be difficult to analyze), PPI-GD divides the data space into small chunks and performs polynomial interpolation on each chunk, which is conceptually simpler and more efficient. Although theoretical analysis is the focus of our work, we also discuss potential optimizations and other practical considerations in implementing our algorithm.
    \item We derive the oracle complexity of PPI-GD in the strongly convex setting (\cref{theorem:oracle-complexity}). We remark that PPI-GD improves upon the LPI-GD algorithm proposed in \cite{JadbabaieMakur2024} by achieving an equivalent oracle complexity under weaker assumptions on the order of $d$, as summarized in \cref{tab:oracle_complexity_comparison}. Moreover, in the important regime where $p = O(\mathsf{poly}(n))$ and $\ep = \Theta(\mathsf{poly}(1/n))$, PPI-GD outperforms GD, SGD, and their variants in oracle complexity when the loss function is sufficiently smooth with respect to the data instance (\cref{proposition:p-over-ep-regime}).
    \item Analogously to the above, we establish the oracle complexity of PPI-GD in the non-convex setting (\cref{theorem:oracle-complexity-nonconvex}) and show asymptotic dominance over GD and SGD in the regime of $p = O(\mathsf{poly}(n))$ and $\epsilon = \Theta(\mathsf{poly}(1/n))$ (\cref{proposition:p-over-epsq-regime})\textemdash also see \cref{tab:oracle_complexity_comparison}.
    \item In service of deriving the oracle complexity results above, we make several key observations regarding multivariate polynomial interpolation which have not been covered in the existing literature on numerical analysis. In particular, we establish an error bound for tensor product polynomial interpolants (\cref{theorem:multivariate-interpolation-bound}) by extending the notion of \emph{blended interpolants} from classical analyses of bicubic splines to the general case of arbitrary degree and dimensionality. Our main technical contribution to this end is the recursive decomposition of total interpolation error into a sum of axis-specific iterated error terms in \cref{proposition:multivariate-interpolation-terms}, derived using strong induction on the recursion tree (see \cref{figure:recursion-tree} for an illustrative small case). To leverage this decomposition, we also derive a bound on each iterated error term in \cref{proposition:multivariate-term-bounds}, and show that interchanging the order of partial differentiation and interpolation is immaterial in most cases in \cref{lemma:interchanging-derivative-and-interpolation}.\footnote{We remark that \cite[p. 217--218]{DeBoor1962} has stated in passing that such tensor product generalizations are morally possible in the context of more restricted smoothness classes, but has not analyzed such generalizations in detail.}
\end{enumerate}

\begin{table*}[t]
    \renewcommand{\arraystretch}{1.25}
    \centering
    \caption{Comparison of the oracle complexity of some first-order optimization algorithms, where the loss function satisfies the assumptions outlined in \cref{subsection:assumptions} and $\gamma \in (0, 1/2)$ is an arbitrary constant.}
    \begin{tabular}{llll}
    \toprule
    Algorithm & Required order of $\d$ & Oracle complexity for strongly convex loss & Oracle complexity for non-convex loss \\
    \midrule
    GD & - & $O (n \log (\frac{p}{\ep}))$ \cite{Nesterov2004} & $O (n\frac{p}{\ep^2})$ \cite{Vavasis1993} \\
    SGD & - & $O(\frac{p}{\ep})$ \cite{BottouCurtis2018} & $O((\frac{p}{\ep^2})^2)$ \cite{GhadimiLan2013} \\
    LPI-GD & $O(\log \log(n))$ & $O \bigprn{\exp \iprn{2 \sqrt{\log(n)}} \bigprn{\frac{p}{\ep}}^{d/(2\ell)} \log \bigprn{\frac{p}{\ep}}}$ \cite{JadbabaieMakur2024} & - \\
    \rowcolor{gray!20} PPI-GD & $O(\log^\gamma(n))$ & $O \bigprn{\exp \bigprn{\log^{2 \gamma}(n)} \bigprn{\frac{p}{\ep}}^{d/(2\ell)} \log \bigprn{\frac{p}{\ep}}}$ (\Cref{theorem:oracle-complexity}) & $O \bigl(\exp \bigl(\log^{2 \gamma}(n)\bigr) \bigl(\frac{p}{\ep^2}\bigr)^{1+d/(2\ell)}\bigr)$ (\Cref{theorem:oracle-complexity-nonconvex}) \\
    \bottomrule
    \end{tabular}
    \label{tab:oracle_complexity_comparison}
\end{table*}

\subsection{Related Literature}

There is a wealth of existing literature on first-order optimization, and we provide a survey of the salient developments therein. The notion of oracle complexity introduced in \cite{NemirovskijYudin1983}, defined as the number of gradient calls necessary for a given algorithm to achieve convergence within some absolute error $\varepsilon$, is used as a measure of algorithmic performance agnostic to the objective function under optimization. We refer readers to \cite{Nesterov2004} for further exposition on oracle complexity, and we remark that the convergence rates of GD and SGD for smooth functions in the convex, strongly convex, and non-convex settings are well-known results \cite{Nesterov2004,Vavasis1993,BottouCurtis2018,GhadimiLan2013}. Earlier work implicitly analyzed SGD in the context of stochastic approximation algorithms \cite{RobbinsMonro1951}, and more recent work \cite{Nemirovskietal2009} has analyzed SGD in canonical and robust settings \cite{BottouCurtis2018}. Since then, numerous variations of GD and SGD have been proposed to bolster their performance in a number of scenarios. For example, mini-batch SGD (MBGD) \cite{Dekeletal2012,BottouCurtis2018} improves wall-clock time in parallel computation settings and convergence rate in sparse regression problems \cite{KhiriratFeyzmahdavian2017}, momentum (heavy-ball method) \cite{Polyak1964} and Nesterov acceleration \cite{Nesterov1983} improve convergence rate in, e.g., strongly convex settings, and specialized versions of GD provide performance gains when optimizing low-rank functions \cite{CossonJadbabaie2023a,CossonJadbabaie2023b,JadbabaieMakur2023}. SGD has been augmented by several variance reduction methods originating from Monte Carlo sampling theory \cite[Chapter 9]{KroeseTaimreBotev2011}, most notably stochastic variance reduced gradient (SVRG) \cite{JohnsonZhang2013}, stochastic average gradient (SAG) \cite{LeRouxSchmidtBach2012,SchmidtLeRouxBach2017}, and stochastic dual coordinate ascent (SDCA) \cite{ShalevShwartzZhang2013}. Adaptive learning rates for GD have been proposed, including adaptive gradient (AdaGrad) \cite{DuchiHazanSinger2011} and its generalization to non-convex objectives \cite{ChakrabartiChopra2021}, root mean squared propagation (RMSprop) \cite{TielemanHinton2012}, and adaptive moments (Adam) \cite{KingmaBa2015}. We refer readers to \cite{Bubeck2015,BottouCurtis2018} for a broader overview of the aforementioned methods, and to \cite{Bishop2006,HastieTibshiraniFriedman2009} for an overview of empirical risk minimization in machine learning. Lastly, classical results on oracle complexity also include the lower bounds in \cite{Nesterov2004,Agarwaletal2009,Fangetal2018,Carmonetal2019,ArjevaniCarmon2023}; we focus exclusively on upper bounds in this work.

Of particular relevance to our present work is the literature on \emph{inexact gradient descent} methods, which includes well-known results in the convex \cite{dAspremont2008,DevolderGlineurNesterov2014} and strongly convex \cite{FriedlanderSchmidt2012,DevolderGlineurNesterov2013} settings. More recently, \cite{RamaswamyBhatnagar2017} analyzes the convergence behavior of GD methods with non-vanishing gradient error, \cite{NecoaraNedelcu2013} utilizes inexact GD methods to study the convergence rate of a dual decomposition problem, and \cite{VlaskiSayed2021} investigates how bias in the gradient oracle assists first-order optimizers in escaping saddle points on non-convex objectives. Inexact GD methods have been generalized to incorporate Nesterov acceleration \cite{QuLi2017,QuLi2019,TrimbachNguyen2022}, 
and rates of convergence have been studied for such extensions. Similarly in spirit to our work, \cite{BhavsarPrashanth2022} introduces \emph{biased gradient oracles} to model the optimization of functions estimated with a batch size parameter, and quantifies the convergence rate of a randomized stochastic gradient algorithm in this regime.

Due to the utility of gradient information in inference tasks such as variable selection, a prior line of work on \emph{gradient estimation} aims to simultaneously learn both the model and the gradient of the loss function, using local polynomial fitting \cite{FanGijbels1996}, local polynomial regression \cite{DeBrabanter2013}, and kernel methods \cite{DelecroixRosa1996}, among others. More recent work has utilized gradient estimation based on methods such as nearest neighbors to design zeroth-order (gradient-free) optimizers \cite{WangDu2018,AussetClemencon2021}. Lastly, first-order methods which use gradient estimation include the recently proposed LPI-GD \cite{JadbabaieMakur2024}, as well as techniques which leverage smoothness of loss functions for federated optimization \cite{JadbabaieMakurShah2023}. We remark that work in this vein often assumes that $d$ grows moderately with respect to $n$ (e.g., $d = O(\log n)$), due to the \emph{curse of dimensionality} as described in \cite[Section 7.1]{FanGijbels1996}. This assumption holds for practical applications in domains such as healthcare \cite{ChoiXiao2018} and control systems \cite{NechybaXu1994}, or problems where dimensionality reduction methods such as principal component analysis \cite{Hotelling1933}, Laplacian eigenmaps \cite{BelkinNiyogi2001}, modal decompositions \cite{Huangetal2024,Makur2019}, 
or isometric embedding theorems such as the \emph{Johnson-Lindenstrauss lemma} \cite{DasguptaGupta2003} may be applied before performing ERM.

\subsection{Outline}

Finally, we delineate the structure of our paper. In \cref{section:formal-model-and-setup}, we introduce applicable notation, formally define the ERM optimization problem under consideration, and specify the assumptions we impose in our analysis of this problem. In \cref{subsection:algorithm}, we formally present the PPI-GD algorithm for solving the ERM optimization problem. Auxiliary results concerning polynomial interpolation error are presented in \cref{subsection:interpolation-error-bounds} in preparation for the ensuing discussion on PPI-GD's oracle complexity in \cref{subsection:oracle-complexity}. We defer formal proofs of the polynomial interpolation error bounds and oracle complexity results to \cref{section:proofs-of-interpolation-error-bounds,section:proofs-of-oracle-complexity}, respectively. Lastly, we provide some numerical simulations of PPI-GD in \cref{section:experiments} and suggest directions for future inquiry in \cref{section:conclusion}.

\section{Formal Model and Setup} \label{section:formal-model-and-setup}

\subsection{Notation}

Let $\positiveintegers$ denote the natural numbers starting from $1$ and let $\nonnegativeintegers = \positiveintegers \cup \ibrc{0}$. Let $\ibktz{a, b} = [a, b] \cap \mathbb{Z}$ and $[a] = \ibktz{1, a}$ denote integer intervals. Lowercase and uppercase bold letters denote vectors and matrices, respectively. Uppercase calligraphic letters denote sets unless otherwise stated.   %
Uppercase non-italic letters denote operators. Given a set $\mathcal{S}$, let $\powerset{\mathcal{S}}$ denote its power set. Given a collection of sets $\mathcal{S}_i$ indexed by $i \in [d]$, let $\bigtimes_{i \in [d]} \mathcal{S}_i = \mathcal{S}_1 \times \cdots \times \mathcal{S}_d$ denote their Cartesian product. Let $\mathcal{C}^p(\mathcal{X})$ denote the class of all functions $f: \mathcal{X} \rightarrow \mathbb{R}$ which are $p$ times continuously differentiable (i.e., all $p$th order partial derivatives exist and are continuous on $\mathcal{X}$). In the context of Landau notation, we use $\tilde{O}(\cdot)$ to hide subpolynomial factors, i.e., factors which are asymptotically dominated by \emph{every} function $n^\beta$ for $\beta > 0$. We use $\Iv{\cdot}$ to denote the Iverson bracket, which equals 1 if the input proposition is true and 0 otherwise. Let $\inorm{\cdot}_p$ denote the $\mathcal{L}^p$-norm on vectors or functions for $p \in [1, \infty]$.

The standard multi-index notation is used for conciseness in the multivariate setting. Namely, given a $d$-tuple multi-index $\vec{s} = (s_1, \dots, s_d) \in \nonnegativeintegers^d$, we let $|\vec{s}| = s_1 + \dots + s_d$, $\vec{x}^\vec{s} = x_1^{s_1} \cdots x_d^{s_d}$ for $\vec{x} \in \mathbb{R}^d$, and $\vec{s}! = s_1! \cdots s_d!$. In addition, given a function $f: \mathbb{R}^d \to \mathbb{R}$, we let $\deriv{f}{\vec{s}} = (\partial^{\abs{\vec{s}}} f) / (\partial x_1^{s_1} \cdots \partial x_d^{s_d})$.

Given a domain $\mathcal{X} \subset \mathbb{R}$, a function $h: \mathcal{X} \rightarrow \mathbb{R}$, and $n + 1$ points $\brc{z_i}_{i=1}^{n+1} \subset \mathcal{X}$ in the domain, let $h[z_1, \dots, z_{n+1}]$ denote the \emph{$n$th divided difference} of $h$ at $\brc{z_i}_{i=1}^{n+1}$, given by the recurrence \cite[p. 308]{AscherGreif2011}
\begin{equation}
    \forall i \!\leq\! j, h[z_i, \dots, z_j] \!=\!\! \begin{cases}
        \!h(z_i), & \!\!\!\!\!\text{if $i = j$}, \\
        \!\frac{h[z_{i+1}, \dots, z_j] - h[z_i, \dots, z_{j-1}]}{z_j - z_i}, & \!\!\!\!\!\text{otherwise}.
    \end{cases} \label{eq:divided-difference}
\end{equation}

\subsection{Problem Statement}

Let $\p \in \positiveintegers$ be the parameter space dimension and $\d \in \positiveintegers$ be the data space dimension. Let $\ibrc{\mathbf{x}^{(i)}: i \in [n]} \subset \mathcal{X}$ be a training set of $n$ data samples, where $\mathcal{X}$ is a bounded data space which we take to be $\mathcal{X} = [0, 1]^d$ for simplicity. Note that we let each data point be a single vector. When data points are given as a pair of vectors $(\vec{a}^{(i)},\vec{b}^{(i)})\in\mathbb{R}^{p_a}\times\mathbb{R}^{p_b}$ (e.g., in supervised learning, where $\vec{b}^{(i)}$ would be the label), we take $\vec{x}$ to be the entire pair, i.e., $\vec{x}^{(i)}=(\vec{a}^{(i)},\vec{b}^{(i)})\in\mathbb{R}^{p_a+p_b}$. Let $f: \mathcal{X} \times \mathbb{R}^p \rightarrow \mathbb{R}$ be a loss function. The loss for a data sample $\vec{x}^{(i)}$ and a parameter vector $\vec{\theta} \in \mathbb{R}^p$ is $f(\vec{x}^{(i)}; \vec{\theta})$, and the ERM objective function $F: \mathbb{R}^\p \to \mathbb{R}$ is the empirical average of the loss over the training set:
\begin{equation}
    F(\vec{\theta}) \define \frac{1}{n} \sum_{i=1}^n f \mprn{\vec{x}^{(i)}; \vec{\theta}} \, .
\end{equation}
We assume that the infimum of the objective function, $F^* \define \inf_{\vec{\theta} \in \mathbb{R}^\p} F(\vec{\theta})
$, exists and is finite. In the strongly convex setting, we consider an algorithm $\mathrm{A}$ which finds an $\ep$-approximation of the optimum $F^*$, i.e., $\mathrm{A}$ returns a parameter vector $\vec{\theta}^* \in \mathbb{R}^p$ which satisfies
\begin{equation}
    F(\vec{\theta}^*) - F^* \leq \ep \, , \label{eq:ep-approximate}
\end{equation}
for a given $\ep > 0$.
In the non-convex setting where finding the global minimum can be intractable, we instead require the algorithm $\mathrm{A}$ to use the notion of an $\ep$-stationary point, i.e., it returns a parameter vector $\vec{\theta}^* \in \mathbb{R}^p$ satisfying
\begin{equation}
    \norm{\grad F(\vec{\theta}^*)}_2 \leq \ep\, . \label{eq:ep-stationary}
\end{equation}

To this end, the algorithm is provided with a first-order oracle $\mathcal{O}: \mathcal{X} \times \mathbb{R}^p \to \mathbb{R}$ \cite{NemirovskijYudin1983}, which returns the gradient of $f$ at a given point $\vec{x} \in \mathcal{X}$ in the data space and a given parameter vector $\vec{\theta} \in \mathbb{R}^p$: $\mathcal{O}(\vec{x}, \vec{\theta}) = \vec{\nabla}_{\vec{\theta}} f \mprn{\vec{x}; \vec{\theta}}$.
We remark that computing the full gradient of $F$ requires $n$ oracle calls.

The \emph{first-order oracle complexity} $\Gamma(\mathrm{A})$ of an algorithm $\mathrm{A}$ is the minimum number of oracle calls necessary for $\mathrm{A}$ to find an $\ep$-approximate solution as defined in \cref{eq:ep-approximate} (or an $\ep$-stationary solution as in \cref{eq:ep-stationary}). This notion of algorithmic complexity, first suggested by \cite{NemirovskijYudin1983}, provides an accurate representation of the algorithm's performance when gradient computation is the bottleneck, as is often the case in practice.

\subsection{Assumptions} \label{subsection:assumptions}

Throughout this work, we make the following assumptions on the loss function which are standard in the optimization literature and establish our theoretical oracle complexity results under these assumptions.
\begin{enumerate}
    \item \emph{(Smoothness in parameters)} There exists $\lipschitzconst > 0$ such that for all $\vec{x} \in \mathcal{X}$, the gradient $\vec{\nabla}_{\vec{\theta}} f(\vec{x}; \vec{\theta})$ is $\lipschitzconst$-Lipschitz continuous with respect to $\vec{\theta}$:
    \begin{equation}
    \displayindent0pt
        \forall \vec{\theta}_1, \vec{\theta}_2 \in \mathbb{R}^p,  \norm{\vec{\nabla}_{\vec{\theta}} f(\vec{x}; \vec{\theta}_1)\!-\!\vec{\nabla}_{\vec{\theta}} f(\vec{x}; \vec{\theta}_2)}_2 \!\leq\! \lipschitzconst \norm{\vec{\theta}_1\!-\!\vec{\theta}_2}_2 .
    \end{equation}
    \item \emph{(H{\"o}lder smoothness in data)} There exist $\holderorder \in \positiveintegers$ and $\holderconst \geq 0$ such that for all $\vec{\theta} \in \mathbb{R}^p$ and $i \in [p]$, $g_i(\vec{x}) = \frac{\partial f}{\partial \theta_i}(\vec{x}; \vec{\theta})$ is $\ell$ times differentiable and belongs in the $\mathcal{H}(\holderorder, \holderconst)$-H{\"o}lder class:
    \begin{align}
        &\forall \mathbf{s} \in \nonnegativeintegers^d \text{ such that } \abs{\mathbf{s}} = \holderorder - 1, \enspace \forall \mathbf{x}_1, \mathbf{x}_2 \in \mathcal{X}, \\
        &\qquad |\deriv{g_i}{\vec{s}}(\mathbf{x}_1) - \deriv{g_i}{\vec{s}}(\mathbf{x}_2)| \leq \holderconst \norm{\mathbf{x}_1 - \mathbf{x}_2}_1 \, . \label{eq:holder-class}
    \end{align}
\end{enumerate}
When we analyze PPI-GD in the strongly convex setting, we additionally make the following assumption:
\begin{enumerate}
    \item[3)] \emph{(Strong convexity in parameters)} There exists $\mu > 0$ such that for all $\vec{x} \in \mathcal{X}$, $f(\vec{x}; \vec{\theta})$ is $\mu$-strongly convex with respect to $\vec{\theta}$:
    \begin{multline}
        \forall \vec{\theta}_1, \vec{\theta}_2 \in \mathbb{R}^p, \quad f(\vec{x}; \vec{\theta}_1) \geq f(\vec{x}; \vec{\theta}_2) \\
         + \vec{\nabla}_{\vec{\theta}} f(\vec{x}; \vec{\theta}_2)\tsp \prn{\vec{\theta}_1 - \vec{\theta}_2} + \frac{\mu}{2} \norm{\vec{\theta}_1 - \vec{\theta}_2}_2^2 \, .
    \end{multline}
\end{enumerate}

Assumptions 1 and 3 are widely adopted in the analysis of gradient descent methods \cite{Nesterov2004,BottouCurtis2018,JohnsonZhang2013,CossonJadbabaie2023b} (and impossibility results such as \cite{KornowskiShamir2021} suggest that such assumptions are required in order to obtain any theoretical guarantee), and the H\"{o}lder class used in Assumption 2 is standard in the literature to describe a function's higher-order smoothness \cite{Tsybakov2009,SadhanalaWang2017,ChenShah2018,JadbabaieMakurShah2023,JadbabaieMakur2024}. Note that we assume integer $\ell$ for simplicity. Intuitively, H{\"o}lder classes roughly capture how closely a function matches its Taylor polynomials of a certain degree. Our definition of the multivariate H{\"o}lder class is a commonly adopted generalization of the usual univariate definition in \cite{Tsybakov2009} (see \cite{Xu2018,SadhanalaWang2017,AgarwalShah2021}).   %
Note that $\mathcal{H}(\holderorder, \holderconst)$-H{\"o}lder smoothness implies
\begin{equation}
    \forall i \in [p], \enspace \forall \mathbf{s} \in \nonnegativeintegers^d \text{ with } \abs{\mathbf{s}} = \holderorder, \enspace \sup_{\vec{x} \in \mathcal{X}} |\deriv{g_i}{\vec{s}}(\vec{x})| \leq \holderconst \, . \label{eq:holder-to-sobolev}
\end{equation}

\textbf{Example Problems that Satisfy Assumptions:}
Finally, we present some examples of problems known to satisfy the assumptions delineated above, including the implicit restriction on the data dimension $d$ (as mentioned in column 2 of \cref{tab:oracle_complexity_comparison}).
\begin{enumerate}[leftmargin=*]
    \item $\mathcal{L}^2$-regularized linear regression (also known as ridge regression) with regularization parameter $\lambda$: Assumptions 1--3 with $\lipschitzconst = 2d + \lambda - 2$, $\holderorder \geq 2$, $\holderconst = 4 \cdot \Iv{\holderorder = 2}$, and $\mu = \lambda$; see \cite[Appendix B.1]{JadbabaieMakur2024} for the $\ell = 2$ case.
    \item Principal component regression \cite{MardiaKentTaylor2024}: Assumptions 1--2 with $L_l = 2d - 2$, $\holderorder=2$, $\holderconst=4$, and small $d$ assumption.
    \item $\mathcal{L}^2$-regularized logistic regression: Assumptions 1--3 with $\lipschitzconst = (d - 1)/2 + \lambda$, $\holderorder \geq 2$, $\holderconst = 3 (\holderorder!)$, and $\mu = \lambda$; see \cite[Appendix B.2]{JadbabaieMakur2024} for the $\ell = 2$ case.\footnote{The factorial scaling of $L_h$ is a consequence of standard growth rates on \emph{Eulerian numbers}, which appear in the analytical expressions for higher-order partial derivatives for logistic regression.}
    \item Any infinitely differentiable loss function (e.g., a feedforward neural network with smooth activation functions) with bounded parameters and data satisfies Assumptions 1--2 for any given $\ell$.
    \item Any problem that is amenable to dimensionality reduction (e.g., problems that exhibit low-rankness in data): Dimensionality reduction techniques, such as principal component analysis, can be used to satisfy the small $d$ assumption. In particular, low-rankness has been empirically observed in many neural network problems \cite{Gur-AriRoberts2018,SagunEvci2018,Papyan2018,SinghBachmannHofmann2021,LeJegelka2022}, and we note that established links between low-rankness and H{\"o}lder smoothness exist in the literature, e.g., \cite{Chatterjee2015,Xu2018,JadbabaieMakurShah2023}.
\end{enumerate}

\section{Main Results and Discussion}

\subsection{Algorithm} \label{subsection:algorithm}

The standard method to find $F^*$ is through gradient descent (GD), which repeats the following iteration until convergence starting from the initial iterate $\vec{\theta}^{(0)}$:
$\vec{\theta}^{(i)} = \vec{\theta}^{(i-1)} - \alpha_t \nabla_\theta F(\vec{\theta}^{(i-1)})$.
The key idea of PPI-GD is to approximate each $\nabla_\vec{\theta} \fsmall{\vec{x}^{(i)}}{\vec{\theta}^{(i-1)}}$ by interpolating over grid points across the data space. If the number of grid points is less than the number of data points $n$, each iteration of PPI-GD will use fewer oracle calls compared to GD.

\makeatletter   %
\newcommand{\algmargin}{\the\ALG@thistlm}
\makeatother
\newlength{\whilewidth}
\settowidth{\whilewidth}{\algorithmicwhile\ }
\algdef{SE}[parWHILE]{parWhile}{EndparWhile}[1]
  {\parbox[t]{\dimexpr\linewidth-\algmargin}{%
     \hangindent\whilewidth\strut\algorithmicwhile\ #1\ \algorithmicdo\strut}}{\algorithmicend\ \algorithmicwhile}%
\algnewcommand{\parState}[1]{\State%
  \parbox[t]{\dimexpr\linewidth-\algmargin}{\strut #1\strut}}

\begin{algorithm}[t]
    \caption{PPI-GD}
    \label{alg:ppigd}
    \begin{algorithmic}[1]
        \Require Training data $\mathcal{D} = \ibrc{\vec{x}^{(i)}: i \in [n]} \subset \mathcal{X} = [0, 1]^d$
        \Require Oracle access to $\grad_\vec{\theta} f(\vec{x}; \vec{\theta})$ at any given $\vec{x} \in \mathcal{X}$ and $\vec{\theta} \in \mathbb{R}^p$
        \Ensure $\vec{\theta}^*$ satisfying $|F(\vec{\theta}^*) - F^*| \leq \ep$
        \State Initial parameters $\vec{\theta}^{(0)} \gets $ some vector $\in \mathbb{R}^\p$;
        \State Set number of iterations $T \in \positiveintegers$ according to \cref{eq:t-value};
        \State Choose the smallest $m \in \positiveintegers$ that satisfies \cref{eq:m-constraint};
        \State Construct uniform grid $\mathcal{G}_{m}^d$ according to \cref{eq:grid-def};
        \State Divide $\mathcal{X}$ into patches $\patches$ as in \cref{eq:patch-def};
        \parState{Precompute coefficients of $\psi_{\vec{s}, \vec{y}}$ for all $\vec{s} \in [\ceil{m/\ell}]^d$ and $\vec{y} \in \vec{\mathcal{G}}(P_\vec{s})$ according to \cref{eq:psi-definition};}
        \For{$t \in [T]$}
            \parState{Make $m^d$ oracle calls to get $\ibrc{\grad_\vec{\theta} f(\vec{u}; \vec{\theta}^{(t-1)}): \vec{u} \in \mathcal{G}_m^d}$;}
            \ForAll{$\vec{x}^{(i)} \in \mathcal{D}$}
                \State Find the patch $P_\vec{s}$ where $\vec{x}^{(i)}$ belongs;
                \parState{Compute the estimate $\widehat{\grad_\vec{\theta}} f (\vec{x}^{(i)}; \vec{\theta}^{(t-1)}) =  (\rho_{\vec{s},1} (\vec{x}^{(i)}),\dots,\rho_{\vec{s},p} (\vec{x}^{(i)}))$ according to \cref{eq:rho-construction} using the queried $\ibrc{\grad_\vec{\theta} f (\vec{u}; \vec{\theta}^{(t-1)}): \vec{u} \in \mathcal{G}_m^d}$ and precomputed $\psi_{\vec{s}, \vec{y}}$;}
            \EndFor
            \State $\vec{\theta}^{(t)} \gets \vec{\theta}^{(t-1)} - \frac{1}{L_l n} \sum_{i=1}^n \widehat{\grad_\vec{\theta}} f(\vec{x}^{(i)}; \vec{\theta}^{(t-1)})$;
        \EndFor
        \State \Return $\vec{\theta}^{(T)}$;
    \end{algorithmic}
\end{algorithm}

First, let the uniform grid $\mathcal{G}_m^d \subset [0, 1]^d$ with $m$ points along each axis be defined as follows:
\begin{equation}
    \mathcal{G}_m^d \define \brc{\frac{1}{m} \vec{u}: \vec{u} \in \bktz{0, m - 1}^d} \, , \label{eq:grid-def}
\end{equation}
with $\mathcal{G}_m = \mathcal{G}_m^1$. At each iteration, $m^d$ oracle calls are made to query the gradient of $f$ at each grid point in $\mathcal{G}_m^d$.

PPI-GD divides the data space $[0, 1]^d$ into hypercube patches $\patches$ which cover the space completely but may overlap. Each patch will contain the same number of grid points $(\ell+1)^d$, but one grid point might belong to more than one patch. The number of grid points in each patch is chosen so that the polynomial interpolant of order $\ell$ is well defined.

Let $\mathcal{I}$ be a set of intervals as follows:
\begin{align}
    \mathcal{I} &\define \brc{\frac{\ell}{m} [u-1, u]: u \in \oneton{\ceil{\frac{m}{\ell}}-1}} \cup \brc{\bkt{1-\frac{\ell}{m}, 1 }} \\
    &= \brc{I_1, \dots, I_{\ceil{\uglyfrac{m}{\ell}}}}
\end{align}
where each $I_i$ are the intervals in $\mathcal{I}$ indexed in order of increasing infimum. Then, the set of patches $\patches$ can be defined as the Cartesian power
\begin{equation}
    \patches \define \mathcal{I}^d = \biggl\{P_\vec{s} = \bigtimes_{i \in \oneton{d}} I_{s_i}: \vec{s} \in \bkt{\ceil{\frac{m}{\ell}}}^d\biggr\} \, . \label{eq:patch-def}
\end{equation}
Let $\mathcal{G}(I_i) = \mathcal{G}_m \cap I_i$ be the grid points belonging in each interval $I_i$ and $\vec{\mathcal{G}}(P_\vec{s}) = \mathcal{G}_m^d \cap P_\vec{s}$ the grid points in each patch $P_\vec{s}$. Note that $\abs{\mathcal{G}(I_i)} = \ell + 1$ for all $I_i \in \mathcal{I}$ and $\abs{\vec{\mathcal{G}}(P_\vec{s})} = (\ell+1)^d$ for all $P_\vec{s} \in \patches$. In addition,
\begin{equation}
    \sup_{x \in I_i} \Biggl(\prod_{y \in \mathcal{G}(I_i)} (x-y)\Biggr) \leq \prod_{j=1}^{\ell} \frac{j}{m} = \frac{\ell!}{m^\ell} \, .
\end{equation}

For each patch $P_\vec{s}$ and $i \in [p]$, there exists a unique polynomial $\rho_{\vec{s}, i}: \mathbb{R}^d \to \mathbb{R}$ of degree $\ell$ along each dimension that interpolates $\frac{\partial f}{\partial \theta_i}(\cdot; \vec{\theta})$ at each grid point in $P_\vec{s}$:

\begin{equation}
    \exists \rho_{\vec{s},i} (\vec{x}) = \!\! \sum_{\vec{t} \in \bktz{0, \ell}^d}\!\! a_\vec{t} \vec{x}^\vec{t}, \enspace \forall \vec{x} \in \vec{\mathcal{G}}(P_\vec{s}), \enspace \rho_{\vec{s},i} (\vec{x}) = \frac{\partial f}{\partial \theta_i}(\vec{x}; \vec{\theta}) \, .
\end{equation}
This interpolant $\rho_{\vec{s},i}$ can be explicitly constructed in terms of the tensor product Lagrange basis:
\begin{equation}
    \rho_{\vec{s},i} (\vec{x}) = \sum_{\vec{y} \in \vec{\mathcal{G}}(P_\vec{s})} \frac{\partial f}{\partial \theta_i}(\vec{y}; \vec{\theta}) \, \psi_{\vec{s}, \vec{y}}(\vec{x}) \, , \label{eq:rho-construction}
\end{equation}
where
\begin{equation}
   \psi_{\vec{s}, \vec{y}}(\vec{x}) \define \prod_{j=1}^d \Biggl(\prod_{z \in \mathcal{G}(I_{s_j}) \setminus \brc{y_j}} \frac{x_j - z}{y_j - z}\Biggr) \, . \label{eq:psi-definition}
\end{equation}
The construction is particularly simple because the multivariate Lagrange basis $\psi_{\vec{s},\vec{y}}$ satisfies (cf. \cite[Section 10.3]{AscherGreif2011})
\begin{equation}
    \forall \vec{x}, \vec{y} \in \vec{\mathcal{G}}(P_\vec{s}), \enspace \psi_{\vec{s}, \vec{y}}(\vec{x}) = \Iv{\vec{x} = \vec{y}} \, ,
\end{equation}
and so the coefficient of $\psi_{\mathbf{s}, \mathbf{y}}$ in the linear combination in \cref{eq:rho-construction} is merely the value the interpolant should fit at the grid point $\mathbf{y}$. Hence, the construction in \cref{eq:rho-construction} produces a valid interpolant, and a constraint counting argument can prove its uniqueness. We remark that tensor product models have been used in the statistical literature to extend univariate methods (e.g., kernel methods) to higher dimensions \cite[Sections 6.5 and 8.6]{Wasserman2006}. Note that each $\psi_{\vec{s}, \vec{y}}(\vec{x})$ does not depend on $f$ or $\vec{\theta}$, so they can be precomputed once $m$, $d$, and $\ell$ are known.

Thus, the interpolant can be evaluated by computing a linear combination of $\ibrc{\grad_\vec{\theta} f(\vec{u}; \vec{\theta}): \vec{u} \in \mathcal{G}_m^d}$ with precomputed weights. The computed values will approximate the true gradient at each data point $\vec{x}$, and their mean will approximate $\vec{\nabla}_\vec{\theta} F$. At the end of each iteration, this approximation can be used instead of the true gradient to perform the gradient descent update, where $\alpha$ is the (fixed) step size $1/L_l$:
\begin{equation}
    \vec{\theta}^{(t)} = \vec{\theta}^{(t-1)} - \frac{\alpha}{n} \sum_{j=1}^n \bkt{\rho_{\vec{s}, i} \mprn{\vec{x}^{(j)}}: i \in [p]}\tsp \, .
\end{equation}
This process is summarized in \cref{alg:ppigd}. If one or more dimensions are discrete (as is the case for classification tasks), interpolating along such dimensions may be inappropriate. Instead, we generate grid points corresponding to each possible value of the discrete variable(s). Note that this does not apply to our theoretical analysis, which assumes all dimensions to be continuous.

We remark that due to the \emph{curse of dimensionality}, a na{\"i}ve implementation of PPI-GD requires $O(p m^d)$ space to store the gradients of $f$ at the grid points, which becomes intractable as $d$ increases.
One approach to alleviate this is to preprocess the data with \emph{dimensionality reduction} \cite{MaatenPostma2009} methods such as principal component analysis, Laplacian eigenmaps, or modal decompositions before running PPI-GD. The Johnson-Lindenstrauss lemma \cite{DasguptaGupta2003} states that an embedding with $d=O(\log(n))$ must exist, which brings the data dimension much closer to our regime of interest.

In addition, real-life data might be clustered, causing the data to be concentrated in a small subset of the patches. In this case, a simple optimization would be to only evaluate gradients at grid points belonging to dense patches (i.e., patches that contain more data points than grid points), and directly evaluate the gradients of data points in sparser patches. Thus, even when $m^d\gg n$, only at most $n$ (and often far fewer in practice) oracle calls per iteration are necessary.

Another practical concern is that the values of $\alpha$ (the learning rate), $m$, and $\ell$ are defined in \cref{alg:ppigd} in terms of the smoothness of the loss function, which might be unknown. Although the values of $\alpha$, $m$, and $\ell$ in \cref{alg:ppigd} were carefully chosen to facilitate theoretical analysis, the chosen values should not be misconstrued to be optimal.
In practice, \emph{hyperparameter optimization} methods \cite{FeurerHutter2019} such as grid search can be used to empirically find suitable values. For a concrete example, refer to our experiments in \cref{section:experiments}, where we use grid search to find good values of $\ell$ and $m$ for each problem.

Notwithstanding the remarks above, the primary objective of our work is theoretical, and we focus on establishing the oracle complexity of PPI-GD in the following sections. To this end, we first derive error bounds on the aforementioned polynomial interpolation process in \cref{subsection:interpolation-error-bounds}, because our subsequent oracle complexity analysis depends on the accuracy of the interpolated gradients. After establishing the interpolation error bounds, we discuss the oracle complexity of PPI-GD in \cref{subsection:oracle-complexity}.

\subsection{Interpolation Error Bounds} \label{subsection:interpolation-error-bounds}

To begin our discussion of polynomial interpolation error, we introduce some preliminary notions and conventions used throughout this analysis. Let $\mathcal{G} = \ibrc{u_i}_{i=1}^{N+1} \subset [a, b]$ be a grid of evenly spaced abscissae on $[a, b]$, namely, $a = u_1 < \cdots < u_{N+1} = b$ and $u_{i+1} - u_i = \Delta u = (b - a) / N$ for all $i \in [N]$. (Note that $N$ has no relation to the size of the training set in the context of our interpolation error analysis.) For each $i \in [N + 1]$, let $\phi_i: [a, b] \rightarrow \mathbb{R}$ be the univariate polynomial interpolation basis function (cf. \cite[Corollary 2]{DeBoor1962}) satisfying
\begin{equation}
    \forall i' \in \bkt{N + 1}, \enspace \phi_i(u_{i'}) = \begin{cases}
        1 \, , & \text{if $i = i'$} \, , \\
        0 \, , & \text{otherwise} \, ,
    \end{cases}
\end{equation}
which in the Lagrange basis is the degree-$N$ polynomial
\begin{equation}
    \phi_i(x) = \prod_{i' \neq i} \frac{x - u_{i'}}{u_i - u_{i'}} \, .
\end{equation}
Given a function $g: [a, b]^d \rightarrow \mathbb{R}$, let $\mathrm{P} g: [a, b]^d \rightarrow \mathbb{R}$ be the $d$-variate \emph{total interpolant} of $g$ at the lattice of abscissae $\mathcal{G}^d$,
\begin{equation}
    \mathrm{P} g(x_1, \dots, x_d) \!=\!\! \sum_{i_1=1}^{N+1} \cdots \sum_{i_d=1}^{N+1} g(u_{i_1}, \dots, u_{i_d}) \prod_{s=1}^d \phi_{i_s}(x_s) \, , \label{eq:total-interpolant}
\end{equation}
where we elide the dependence on $N$ and $\mathcal{G}$ in the notation $\mathrm{P}$ for simplicity. For any $j \in [d]$, let $\mathrm{P}_j g: [a, b]^d \rightarrow \mathbb{R}$ be the \emph{univariate partial interpolant} of $g$ with respect to $x_j$,
\begin{equation}
    \mathrm{P}_j g(x_1, \dots, x_d) \!=\!\! \sum_{i=1}^{N+1}\! g(x_1, \dots, x_{j-1}, u_i, x_{j+1}, \dots, x_d) \, \phi_{i}(x_j) \, .
\end{equation}
This definition naturally generalizes to \emph{multivariate partial interpolants} by taking multiple subscripts and computing linear combinations over the tensor product of basis interpolants in the subscripted dimensions:\footnote{The $r$-times tensor product of the vector space of degree-$N$ univariate polynomials is a \emph{strict subset} of the vector space of degree-$rN$ $r$-variate polynomials.}
\begin{align}
    &\mathrm{P}_{j_1, \dots, j_r} g(x_1, \dots, x_d) \\
    &\quad = \sum_{i_1 = 1}^{N+1} \cdots \sum_{i_r = 1}^{N+1} g(x_1, \dots, u_{i_s}, \dots, x_d) \prod_{s=1}^r \phi_{i_s}(x_{j_s}) \, . \label{eq:multivariate-partial-interpolant}
\end{align}
When all $d$ dimensions are included in the subscript, we recover the total interpolant $\mathrm{P}_{1, \dots, d} g = \mathrm{P} g$. We mention two remarks regarding the definitions above. Firstly, forming a partial interpolant $\mathrm{P}_{j_1, \dots, j_r} g$ according to \cref{eq:multivariate-partial-interpolant} and evaluating its value at some point $(v_1, \dots, v_d) \in [a, b]^d$ is equivalent to first performing partial application on $g$ to obtain $g(x_{j_1}, \dots, x_{j_r}) = g(v_1, \dots, x_{j_s}, \dots, v_d)$, forming the total interpolant $\mathrm{P} g$ according to \cref{eq:total-interpolant}, and then evaluating the total interpolant at $(v_{j_1}, \dots, v_{j_r})$. Hence, interpolation error bounds for univariate functions in the literature apply \emph{ipso facto} to univariate partial interpolants of multivariate functions. Secondly, evaluating a multivariate interpolant of the form \cref{eq:total-interpolant} at some point $(v_1, \dots, v_d) \in [a, b]^d$ may equivalently be done by successively forming and evaluating univariate partial interpolants (cf. \cite[Section 3.6.2]{PressTeukolsky2007}):
\begin{align}
    &\equad \mathrm{P} g(v_1, \dots, v_d) \\
    &\stackclap{(a)}{=} \sum_{i_1=1}^{N+1} \cdots \sum_{i_d=1}^{N+1} g(u_{i_1}, \dots, u_{i_d}) \prod_{s=1}^d \phi_{i_s}\!(v_s) \\
    &\stackclap{(b)}{=} \sum_{i_d=1}^{N+1}\! \Biggl( \cdots \sum_{i_2=1}^{N+1}\!\Biggl( \sum_{i_1=1}^{N+1}  g(u_{i_1}, \dots, u_{i_d}) \, \phi_{i_1}\!(v_1)\!\! \Biggr)  \phi_{i_2}\!(v_2) \cdots \!\Biggr)  \phi_{i_d}\!(v_d) \\
    &\stackclap{(\dagger_1)}{=} \sum_{i_d=1}^{N+1}\! \Biggl(\cdots \sum_{i_2=1}^{N+1} \mathrm{P}_1 g(v_1, u_{i_2}, \dots, u_{i_d}) \, \phi_{i_2}\!(v_2) \cdots\!\Biggr) \phi_{i_d}\!(v_d) \\
    &\stackclap{(\dagger_2)}{=} \sum_{i_d=1}^{N+1}\! \Biggl( \cdots \mathrm{P}_2 \mathrm{P}_1 g(v_1, v_2, u_{i_3}, \dots, u_{i_d}) \cdots \!\Biggr) \phi_{i_d}\!(v_d) \\
    &=\cdots\stackclap{(\dagger_d)}{=} \mathrm{P}_d \cdots \mathrm{P}_1 g(v_1, \dots, v_d) \, , \label{eq:total-interpolant-characterization}
\end{align}
where (a) holds by \cref{eq:total-interpolant} and (b) holds by reversing the order of the summations and applying the distributive property. Algorithmically, we initialize a $d$-dimensional array of values of $g$ evaluated at the abscissae $\mathcal{G}^d$. In step ($\dagger_1$), we use this array to construct a univariate interpolant along the $x_1$ axis for each $(u_{i_2}, \dots, u_{i_d}) \in \mathcal{G}^{d-1}$, and evaluate all the interpolants at $v_1$ to obtain a $(d - 1)$-dimensional array. In step ($\dagger_2$), we use this new array to construct a univariate interpolant along the $x_2$ axis for each $(u_{i_3}, \dots, u_{i_d}) \in \mathcal{G}^{d-2}$, and evaluate all the interpolants at $v_2$ to obtain a $(d - 2)$-dimensional array. Repeating this process until step ($\dagger_d$) yields the desired result.

For any sequence $j_1, \dots, j_r \in [d]$ of dimensions, let $\mathrm{E}_{j_1, \dots, j_r} g: [a, b]^d \rightarrow \mathbb{R}$ be the \emph{iterated interpolation error} given by
\begin{equation}
    \mathrm{E}_{j_1, \dots, j_r} g \!= \!\begin{cases}
        \mathrm{P}_{j_1} g - g \, , & \!\!\text{if $r = 1$}  , \\
        \mathrm{P}_{j_r} \mathrm{E}_{j_1, \dots, j_{r-1}} g - \mathrm{E}_{j_1, \dots, j_{r-1}} g \, , &\!\! \text{otherwise}  .
    \end{cases} \label{eq:iterated-interpolation-error}
\end{equation}
Note that multiple subscripts of $\mathrm{E}$ in \cref{eq:iterated-interpolation-error} specify \emph{successive} iterations of \emph{uni}variate partial interpolation, each on the $\mathrm{E}$ term from the preceding iteration, while multiple subscripts of $\mathrm{P}$ in \cref{eq:multivariate-partial-interpolant} specify \emph{one} iteration of \emph{multi}variate partial interpolation on $g$.\footnote{Equivalently, multiple subscripts of $\mathrm{P}$ specify successive iterations of univariate partial interpolation, each on the $\mathrm{P}$ term from the preceding iteration, as per the discussion in \cref{eq:total-interpolant-characterization}.}

With these preliminaries established, we present the main result of this subsection, an upper bound on the uniform norm interpolation error for $d$-variate functions.

\begin{theorem}[Polynomial interpolation error bound] \label{theorem:multivariate-interpolation-bound}
    Let $g: [a, b]^d \rightarrow \mathbb{R}$ be a function in the H{\"o}lder class $\mathcal{H}(\holderorder, \holderconst)$. Assume that $\holderorder \geq d$ and $\abs{\mathcal{G}} = N + 1 \geq \holderorder$. Then, the error in the total interpolant of $g$ satisfies
    \begin{equation}
        \norm{\mathrm{P} g - g}_\infty \leq \holderconst \prn{\Delta u}^\holderorder \prn{\frac{2^{N}}{N} + 1}^d \, .
    \end{equation}
\end{theorem}

\cref{theorem:multivariate-interpolation-bound} is proved in \cref{section:proofs-of-interpolation-error-bounds}. We split our argument into three steps, delineated by the following propositions which are also proved in \cref{section:proofs-of-interpolation-error-bounds}. Firstly, we decompose the total interpolation error into a sum of iterated interpolation error terms, with one iterated term for each non-empty subset of the $d$ variables.

\begin{proposition} \label{proposition:multivariate-interpolation-terms}
    For any continuous function $g: [a, b]^d \rightarrow \mathbb{R}$, the error in the total interpolant of $g$ satisfies
    \begin{equation}
        \norm{\mathrm{P} g - g}_\infty \leq \sum_{\substack{\mathcal{K} \in \powerset{[d]}: \\ \mathcal{K} \neq \emptyset}} \norm{\mathrm{E}_\mathcal{K} g}_\infty \, .
    \end{equation}
\end{proposition}

To establish the intuition behind \cref{proposition:multivariate-interpolation-terms}, we consider an example where $g$ is a function of three variables $x_1$, $x_2$, and $x_3$. \cref{figure:recursion-tree} illustrates the decomposition in this case, and the steps labeled with letters below correspond to edges in the figure. Using the triangle inequality, we upper-bound $\norm{\mathrm{P} g - g}_\infty = \norm{\mathrm{P}_{1,2,3} g - g}_\infty$ as a sum of the magnitudes of three partial interpolants of error terms, where each partial interpolant is subscripted with a different suffix of the variables $[x_1, x_2, x_3]$:
\begin{align}
    &\equad \norm{\mathrm{P}_{1,2,3} g - g}_\infty \\
    &\leq \norm{\mathrm{P}_{1,2,3} g - \mathrm{P}_{2,3} g}_\infty + \norm{\mathrm{P}_{2,3} g - \mathrm{P}_3 g}_\infty + \norm{\mathrm{P}_3 g - g}_\infty \\
    &\stackclap{(a)}{=} \norm{\mathrm{P}_{2,3} \mathrm{E}_1 g}_\infty + \norm{\mathrm{P}_3 \mathrm{E}_2 g}_\infty + \norm{\mathrm{E}_3 g}_\infty \, . \label{eq:proof-sketch-step-a}
\end{align}
The partial interpolants $\mathrm{P}_{2,3} g$ and $\mathrm{P}_3 g$ in the triangle inequality are analogous to \emph{blended interpolants} from the bicubic spline literature \cite{CarlsonHall1973}. Next, we split each term in \cref{eq:proof-sketch-step-a} into the sum of an iterated interpolation error of $g$ and another term to be recursively decomposed:
\begin{align}
    \norm{\mathrm{P}_{2,3} \mathrm{E}_1 g}_\infty &\stackclap{(b)}{\leq} \norm{\mathrm{E}_1 g}_\infty + \norm{\mathrm{P}_{2,3} \mathrm{E}_1 g - \mathrm{E}_1 g}_\infty \, , \label{eq:proof-sketch-step-b} \\
    \norm{\mathrm{P}_3 \mathrm{E}_2 g}_\infty &\stackclap{(c)}{\leq} \norm{\mathrm{E}_2 g}_\infty + \norm{\mathrm{P}_3 \mathrm{E}_2 g - \mathrm{E}_2 g}_\infty \, , \\
    \norm{\mathrm{E}_3 g}_\infty &= \norm{\mathrm{E}_3 g}_\infty \, .
\end{align}
Next, we recursively decompose $\norm{\mathrm{P}_{2,3} \mathrm{E}_1 g - \mathrm{E}_1 g}_\infty$. Similarly to \cref{eq:proof-sketch-step-a}, we obtain
\begin{align}
    \norm{\mathrm{P}_{2,3} \mathrm{E}_1 g - \mathrm{E}_1 g}_\infty &\stackclap{(d)}{\leq} \norm{\mathrm{P}_3 \mathrm{E}_{1,2} g}_\infty + \norm{\mathrm{E}_{1,3} g}_\infty \, ,
\end{align}
and similarly to \cref{eq:proof-sketch-step-b}, we obtain
\begin{align}
    \norm{\mathrm{P}_3 \mathrm{E}_{1,2} g}_\infty &\stackclap{(e)}{\leq} \norm{\mathrm{E}_{1,2} g}_\infty + \norm{\mathrm{P}_3 \mathrm{E}_{1,2} g - \mathrm{E}_{1,2} g}_\infty \, , \\
    \norm{\mathrm{E}_{1,3} g}_\infty &= \norm{\mathrm{E}_{1,3} g}_\infty \, .
\end{align}
Lastly, by definition of $\mathrm{E}$, we have
\begin{align}
    \norm{\mathrm{P}_3 \mathrm{E}_2 g - \mathrm{E}_2 g}_\infty &\stackclap{(f)}{=} \norm{\mathrm{E}_{2,3} g}_\infty \, , \\
    \norm{\mathrm{P}_3 \mathrm{E}_{1,2} g - \mathrm{E}_{1,2} g}_\infty &\stackclap{(g)}{=} \norm{\mathrm{E}_{1,2,3} g}_\infty \, .
\end{align}
In a nutshell, by repeatedly applying the logic of steps \cref{eq:proof-sketch-step-a} and \cref{eq:proof-sketch-step-b}, the recursive decomposition above generates an iterated error term $\inorm{\mathrm{E}_\mathcal{K} g}_\infty$ for every non-empty subset of variables $\mathcal{K}$. Hence, this process may be modeled with a recursion tree, and using strong induction over the tree proves \cref{proposition:multivariate-interpolation-terms}. The power of $d$ scaling within the power set $2^{[d]}$ is expected due to the curse of dimensionality.

\begin{figure}
    \centering
    
    \adjustbox{scale=0.7}{\begin{tikzpicture}[auto]
        \node (n0-0) [state, rectangle, rounded corners, blue] at (0.0, 0.0) {$\norm{\mathrm{P}_{1,2,3} g - g}_\infty$};
        \node (n1-0) [state, rectangle, rounded corners, red] at (-3.5, -1.5) {$\norm{\mathrm{P}_{2,3} \mathrm{E}_1 g}_\infty$};
        \node (n1-1) [state, rectangle, rounded corners, red] at (0.0, -1.5) {$\norm{\mathrm{P}_3 \mathrm{E}_2 g}_\infty$};
        \node (n1-2) [state, rectangle, rounded corners] at (3.0, -1.5) {$\norm{\mathrm{E}_3 g}_\infty$};
        \node (n2-0) [state, rectangle, rounded corners] at (-7.0, -3.0) {$\norm{\mathrm{E}_1 g}_\infty$};
        \node (n2-1) [state, rectangle, rounded corners, blue] at (-3.5, -3.0) {$\norm{\mathrm{P}_{2,3} \mathrm{E}_1 g - \mathrm{E}_1 g}_\infty$};
        \node (n2-2) [state, rectangle, rounded corners] at (0.0, -3.0) {$\norm{\mathrm{E}_2 g}_\infty$};
        \node (n2-3) [state, rectangle, rounded corners, blue] at (3.0, -3.0) {$\norm{\mathrm{P}_3 \mathrm{E}_2 g - \mathrm{E}_2 g}_\infty$};
        \node (n3-0) [state, rectangle, rounded corners, red] at (-5.0, -4.5) {$\norm{\mathrm{P}_3 \mathrm{E}_{1,2} g}_\infty$};
        \node (n3-1) [state, rectangle, rounded corners] at (-2.0, -4.5) {$\norm{\mathrm{E}_{1,3} g}_\infty$};
        \node (n3-2) [state, rectangle, rounded corners] at (3.0, -4.5) {$\norm{\mathrm{E}_{2,3} g}_\infty$};
        \node (n4-0) [state, rectangle, rounded corners] at (-6.5, -6.0) {$\norm{\mathrm{E}_{1,2} g}_\infty$};
        \node (n4-1) [state, rectangle, rounded corners, blue] at (-3.5, -6.0) {$\norm{\mathrm{P}_3 \mathrm{E}_{1,2} g - \mathrm{E}_{1,2} g}_\infty$};
        \node (n5-0) [state, rectangle, rounded corners] at (-3.5, -7.5) {$\norm{\mathrm{E}_{1,2,3} g}_\infty$};

        \path [-stealth] (n0-0) edge [blue] node [scale=0.7, above] {(a)} (n1-0);
        \path [-stealth] (n0-0) edge [blue] node [scale=0.7] {(a)} (n1-1);
        \path [-stealth] (n0-0) edge [blue] node [scale=0.7] {(a)} (n1-2);
        \path [-stealth] (n1-0) edge [red] node [scale=0.7, above] {(b)} (n2-0);
        \path [-stealth] (n1-0) edge [red] node [scale=0.7] {(b)} (n2-1);
        \path [-stealth] (n1-1) edge [red] node [scale=0.7] {(c)} (n2-2);
        \path [-stealth] (n1-1) edge [red] node [scale=0.7] {(c)} (n2-3);
        \path [-stealth] (n2-1) edge [blue] node [scale=0.7, left] {(d)} (n3-0);
        \path [-stealth] (n2-1) edge [blue] node [scale=0.7, right] {(d)} (n3-1);
        \path [-stealth] (n2-3) edge [blue] node [scale=0.7] {(f)} (n3-2);
        \path [-stealth] (n3-0) edge [red] node [scale=0.7, left] {(e)} (n4-0);
        \path [-stealth] (n3-0) edge [red] node [scale=0.7, right] {(e)} (n4-1);
        \path [-stealth] (n4-1) edge [blue] node [scale=0.7] {(g)} (n5-0);

        \def\pad{4pt}

        \draw[rounded corners=8pt, dashed, gray]
            let \p0=(n1-0.north east),
                \p1=(n3-1.north east),
                \p3=(n5-0.south),
                \p4=(n2-0.north west),
                \p5=(n1-0.north west),
            in
                (\x0 + \pad, \y0 + \pad)
                -- (\x1 + \pad, \y1 + \pad)
                -- (\x1 + \pad, \y3 - \pad)
                -- (\x4 - \pad, \y3 - \pad)
                -- (\x4 - \pad, \y4 + \pad)
                -- (\x5 - \pad, \y5 + \pad)
                -- cycle
            node[below, align=center] at (\x4 / 2 + \x1 / 2, \y3 - \pad) {\footnotesize $\norm{\mathrm{E}_\mathcal{K} g}_\infty$ terms with $\mathcal{K} \in 2^{\bktz{1,3}}$ and $1 \in \mathcal{K}$};

        \draw[rounded corners=8pt, dashed, gray]
            let \p0=(n1-1.north east),
                \p1=(n1-2.west),
                \p2=(n2-3.north east),
                \p3=(n3-2.south),
                \p4=(n2-2.south west),
                \p5=(n2-2.north west),
                \p6=(n1-1.south west)
            in
                (\x0 + \pad, \y0 + \pad)
                -- (\x1, \y2 + \pad)
                -- (\x2 + \pad, \y2 + \pad)
                -- (\x2 + \pad, \y3 - \pad)
                -- (\x4 - \pad, \y3 - \pad)
                -- (\x4 - \pad, \y5)
                -- (\x6 - \pad, \y6)
                -- (\x6 - \pad, \y0 + \pad)
                -- cycle
            node[below, align=center] at (\x4 / 2 + \x2 / 2, \y3 - \pad) {\footnotesize $\norm{\mathrm{E}_\mathcal{K} g}_\infty$ terms with $\mathcal{K} \in 2^{\bktz{2,3}}$ and $2 \in \mathcal{K}$};
    \end{tikzpicture}}
    \caption{Recursion tree for the three-variable case of \cref{proposition:multivariate-interpolation-terms}, showing how we decompose total interpolation error $\inorm{\mathrm{P} g - g}_\infty = \inorm{\mathrm{P}_{1,2,3} g - g}_\infty$ into a sum of iterated error terms $\inorm{\mathrm{E}_\mathcal{K} g}_\infty$, with one such term for each non-empty subset of variables. Letters along edges correspond to labeled steps from the proof sketch in \cref{subsection:interpolation-error-bounds}. At blue nodes (e.g., \cref{eq:proof-sketch-step-a}), the decomposition process produces partial interpolants of iterated errors. At red nodes (e.g., \cref{eq:proof-sketch-step-b}), the process produces an iterated error term (black child) and another term to be recursively decomposed (blue child). The partial interpolation operator $\mathrm{P}$ in blue nodes is subscripted with a suffix of the variables $[x_1, x_2, x_3]$, and blue nodes with $\mathrm{P}_{\bktz{t, 3}}$ have $3 - t$ red children and $1$ black child.}
    \label{figure:recursion-tree}
\end{figure}
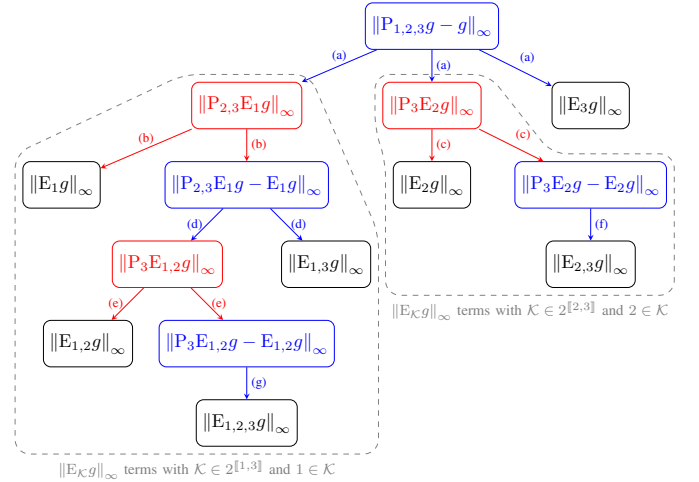

Secondly, we establish an error bound for univariate polynomial interpolation, based on judicious usage of a Newton basis representation of the interpolant followed by application of the mean value theorem for divided differences:

\begin{proposition} \label{proposition:univariate-interpolation-bound}
    Consider any function $h: [a, b] \rightarrow \mathbb{R}$ with $h \in \mathcal{C}^q([a, b])$. Assume $\abs{\mathcal{G}} = N + 1 \geq q$. Then, the error in the (degree-$N$) univariate interpolant of $h$ satisfies
    \begin{align}
        \norm{\mathrm{P} h - h}_\infty \leq \frac{2^{N+1-q}}{N} \prn{\Delta u}^q \norm{h^{(q)}}_\infty \, .
    \end{align}
\end{proposition}

Thirdly, we upper-bound the magnitude of the iterated error terms from \cref{proposition:multivariate-interpolation-terms} by repeatedly applying \cref{proposition:univariate-interpolation-bound} to bound the effect of each iteration:

\begin{proposition} \label{proposition:multivariate-term-bounds}
    Consider any function $g: [a, b]^d \rightarrow \mathbb{R}$ with $g \in \mathcal{C}^\ell([a, b]^d)$. Assume $\ell \geq d$ and $\abs{\mathcal{G}} = N + 1 \geq \ell$. For notational convenience, let $\alpha_q = 2^{N+1-q}/N$ for any $q \in \positiveintegers$. Then, for any distinct $j_1, \dots, j_r \in [d]$,
    \begin{equation}
        \norm{\mathrm{E}_{j_1, \dots, j_r} g}_\infty \!\leq\! \alpha_{\ell-r+1} \, \alpha_1^{r-1}\! \mprn{\Delta u}^\ell\! \norm{\frac{\partial^\ell g}{\partial x_{j_1}^{\ell-r+1} \partial x_{j_2} \cdots\, \partial x_{j_r}}}_\infty \!\!\!\!.
    \end{equation}
\end{proposition}

Each application of \cref{proposition:univariate-interpolation-bound} gives rise to one of the $\alpha$ terms above. The proof of \cref{proposition:multivariate-term-bounds} also utilizes the following lemma, which states the equality of interchanging partial differentiation and partial interpolation and is proved in \cref{section:proofs-of-interpolation-error-bounds}.

\begin{lemma} \label{lemma:interchanging-derivative-and-interpolation}
    For any distinct $j_1, \dots, j_r, k \in [d]$ and any function $g: [a, b]^d \rightarrow \mathbb{R}$ differentiable with respect to $x_k$, we have
    \begin{equation}
        \frac{\partial \mathrm{P}_{j_1, \dots, j_r} g}{\partial x_k} = \mathrm{P}_{j_1, \dots, j_r} \mprn{\frac{\partial g}{\partial x_k}} \, .
    \end{equation}
\end{lemma}

Finally, combining the intermediate results above and using the H\"{o}lder smoothness of $g$ proves \cref{theorem:multivariate-interpolation-bound}. As previously mentioned, we defer the technical details to \cref{section:proofs-of-interpolation-error-bounds}.

\subsection{Oracle Complexity} \label{subsection:oracle-complexity}

Now that we have analyzed the error bounds on multivariate polynomial interpolation in a general setting, we apply those results to our PPI-GD algorithm. We assume that $\|\vec{\theta}^{(0)}-\vec{\theta}^*\|_\infty$ is bounded, i.e., $\|\vec{\theta}^{(0)}-\vec{\theta}^*\|_2^2 \leq c^2p$ for some constant $c>0$ (cf. \cite[Equation (35)]{JadbabaieMakur2024}). First, we get the following result on the error of our approximate gradient.

\begin{lemma} \label{lemma:gradient-error}
For any constant $\delta > 0$, if the number of grid points $m^d$ is large enough that
\begin{equation}
    m \geq \biggl(\frac{\holderconst}{\delta} \biggl(\frac{2^{\holderorder}}{\holderorder} + 1\biggr)^d\biggr)^\uglyfrac{1}{\holderorder} \label{eq:m-constraint}
\end{equation}
holds, then for all $i \in [p]$ and $\vec{\theta} \in \mathbb{R}^p$,
\begin{equation}
    \sup_{\vec{x} \in \mathcal{X}} \abs{\widehat{\frac{\partial f}{\partial \theta_i}}(\vec{x}; \vec{\theta}) - \frac{\partial f}{\partial \theta_i}(\vec{x}; \vec{\theta})} \leq \delta\, ,
\end{equation}
where the approximate gradient $\bigl( \widehat{\frac{\partial f}{\partial \theta_1}}(\vec{x}; \vec{\theta}), \dots, \widehat{\frac{\partial f}{\partial \theta_p}}(\vec{x}; \vec{\theta}) \bigr) \allowbreak= \widehat{\vec{\nabla}_{\vec{\theta}}} f(\vec{x};\vec{\theta})$ is computed with the interpolating polynomial as defined in \cref{eq:rho-construction}.
\end{lemma}
\cref{lemma:gradient-error} is proved in \cref{section:proofs-of-oracle-complexity}. Then, we use the aforementioned result on interpolation error to derive the oracle complexity of our proposed PPI-GD method in the strongly convex setting.

\begin{theorem}[Oracle complexity of PPI-GD in strongly convex setting] \label{theorem:oracle-complexity}
Suppose Assumptions 1 through 3 hold. Given some (small) accuracy $\varepsilon$ satisfying $0<\varepsilon < (L_l-\mu)p/(2L_l \mu)$, let $m$ be the smallest positive integer that satisfies inequality \cref{eq:m-constraint} in \cref{lemma:gradient-error} with $\delta$ given by \cref{eq:delta-value}. Let the step size be $\alpha = 1/\lipschitzconst$. If $\ell \geq d \geq 2$ and $d \leq \log^{-\uglyfrac{1}{2}} (2) \log^\gamma(n)$ for some constant $0 < \gamma < 1/2$,
the first-order oracle complexity of PPI-GD (to obtain an $\varepsilon$-approximate solution) is bounded by
\begin{equation}
    \Gamma(\operatorname{PPI-GD}) \leq C_1 \exp\prn{\log^{2\gamma}(n)} \prn{\frac{p}{\ep}}^{\uglyfrac{d}{2\ell}} \prn{\log\prn{\frac{p}{\ep}} + C_2}
\end{equation}
where the constants 
\begin{equation}
C_1 \triangleq  \frac{(1+\holderconst)(1+c^2L_l \mu)L_l}{\mu(\lipschitzconst-\mu)\log\prn{\frac{L_l}{L_l-\mu}}} , \;
C_2 \triangleq \log\mprn{\!\frac{1+c^2\lipschitzconst \mu}{2\mu}\!} \label{eq:c12def}
\end{equation}
 only depend on $\mu$, $\lipschitzconst$, $\holderconst$, and $c$.
\end{theorem}

\Cref{theorem:oracle-complexity} is proved in \Cref{section:proofs-of-oracle-complexity}. Although this bound is not directly comparable with \cite[Theorem 2]{JadbabaieMakur2024}, which presents oracle complexity bounds for LPI-GD in terms of $d$ instead of $n$, the intermediate bound \cref{eq:intermediate-bound} allows for a direct comparison which reveals an improvement in the scaling in $d$ (polyexponential vs. doubly exponential). We further remark that assumptions such as $\ell\geq d\geq 2$ are not necessarily limitations of the algorithm, but are only required to simplify the bound to a form that intuitively captures its asymptotic qualities. To obtain the final scaling with $p$ in the proof, we assume that the parameters $\vec{\theta}$ belong in a hypercube with constant edge length (that does not depend on $d$). %
This is used to establish a relation between the parameter space dimension $p$ and the $\mathcal{L}^2$-distance between the initial choice of parameters $\vec{\theta}^{(0)}$ and minimizer $\vec{\theta}^*$.

Next, we present the following proposition which illustrates a regime where PPI-GD outperforms GD, SGD, SVRG, and MBGD in terms of oracle complexity for strongly convex loss.
\begin{proposition} \label{proposition:p-over-ep-regime}
    Under the same assumptions as \cref{theorem:oracle-complexity}, if $p / \ep = O(n^\beta)$ for some $\beta > 0$, GD, SGD, SVRG, MBGD, and PPI-GD exhibit the following oracle complexity (expressed in Landau notation):
    
    {\centering
    \vspace{0.5em}
    \begin{tabular}{llll}
    \toprule
    Algorithm & Oracle complexity \\
    \midrule
    \rule{0pt}{11pt}GD & $\tilde{O}(n)$ \\
    SGD & $\tilde{O}(n^\beta)$ \\
    SVRG & $\tilde{O}(n)$ \\
    MBGD & $\tilde{O}(n^\beta)$ \\
    PPI-GD & $\tilde{O}(n^{(d \beta) / (2 \ell)})$\\
    \bottomrule
    \end{tabular}\par
    \vspace{0.5em}
    }
\end{proposition}
\Cref{proposition:p-over-ep-regime} is proved in \cref{section:proofs-of-oracle-complexity}. The regime $p / \ep = O(n^\beta)$ studied here encompasses a wide range of common settings, including the case where $p$ and $1/\ep$ are some positive power of $n$. For example, in machine learning contexts where overparameterization ($p>n$) can often be beneficial \cite{ZhangBengio2021}, it is natural to consider a regime of $p=O(n^b)$ (where $b \geq 1$) and $\ep=\Theta(1/\sqrt{n})$, with the error bound matching the scaling of the standard error.

Note that \cref{proposition:p-over-ep-regime} implies the oracle complexity of PPI-GD to be $\tilde{O}(n^{\beta/2})$ since $\ell \geq d$ by assumption. Thus, in this regime, PPI-GD always outperforms the other methods in terms of asymptotic oracle complexity when $\beta < 2$. In addition, increasing $\ell$ can reduce the exponent $d\beta/(2\ell)$ to any arbitrary positive value, which implies that sufficient smoothness in data allows PPI-GD to outperform the other methods for \emph{any} $\beta > 0$.

Although the strongly convex setting is important and well studied in the optimization literature, strong convexity can be a stringent requirement. Thus, we now shift our attention to the non-convex case. Note that the step size must be reduced from $1/\lipschitzconst$ to $1/(4\lipschitzconst)$ in this setting.

\begin{theorem}[Oracle complexity of PPI-GD in non-convex setting] \label{theorem:oracle-complexity-nonconvex}
    Given some (small) $\varepsilon$ satisfying $0<\varepsilon<\sqrt{p}$
    , let $m$ be the smallest positive integer that satisfies inequality \cref{eq:m-constraint} in \cref{lemma:gradient-error} with $\delta = \varepsilon/(2\sqrt{p})$. Let the step size be $\alpha = 1/(4\lipschitzconst)$. If $\ell \geq d \geq 2$ and $d \leq \log^{-\uglyfrac{1}{2}} (2) \log^\gamma(n)$ for some constant $0 < \gamma < 1/2$,
    the first-order oracle complexity of PPI-GD (to reach an $\varepsilon$-stationary point) is bounded by
\begin{equation}
    \Gamma(\operatorname{PPI-GD}) \leq C_3 \exp\prn{\log^{2\gamma}(n)} \prn{\frac{p}{\varepsilon^2}}^{1+d/(2\ell)}
\end{equation}
where the constant
$
C_3 \triangleq 2(1+32 c^2 L_l^2)(1+L_h)
$
only depends on $\lipschitzconst$, $\holderconst$, and $c$.
\end{theorem}
An analog to \cref{proposition:p-over-ep-regime} immediately follows.
\begin{proposition} \label{proposition:p-over-epsq-regime}
    Under the same assumptions as \cref{theorem:oracle-complexity-nonconvex}, if $p / \ep^2 = O(n^\beta)$ for some $\beta > 0$, GD, SGD, and PPI-GD exhibit the following oracle complexity (expressed in Landau notation):
    
    {\centering
    \vspace{0.5em}
    \begin{tabular}{llll}
    \toprule
    Algorithm & Oracle complexity \\
    \midrule
    \rule{0pt}{11pt}GD & $\tilde{O}(n^{1+\beta})$ \\
    SGD & $\tilde{O}(n^{2\beta})$ \\
    PPI-GD & $\tilde{O}(n^{(1+d/(2\ell))\beta})$\\
    \bottomrule
    \end{tabular}\par
    \vspace{0.5em}
    }
\end{proposition}
\Cref{theorem:oracle-complexity-nonconvex} and \cref{proposition:p-over-epsq-regime} are proved in \cref{section:proofs-of-oracle-complexity}. All our remarks regarding \cref{proposition:p-over-ep-regime} also apply to \cref{proposition:p-over-epsq-regime}: under our assumptions, PPI-GD beats the other methods when $\beta < 2$, and sufficient smoothness (i.e., a large enough $\ell$) can arbitrarily relax this bound.

We conclude this section with some remarks on the implications of our analysis by reexamining the oracle complexities summarized in \cref{tab:oracle_complexity_comparison}. Looking at each factor in isolation, we see that PPI-GD's oracle complexity $O(\exp(\log^{2\gamma}(n)) (\uglyfrac{p}{\ep})^{\uglyfrac{d}{(2\ell)}} \log(\uglyfrac{p}{\ep}))$ for strongly convex loss has better scaling in $n$ than GD's $O(n \log (p/\ep))$ (neglecting other factors like $p$ and $\ep$), and better scaling in $p/\ep$ than SGD's $O(p/\ep)$. Likewise, in the non-convex setting, PPI-GD's $O (\exp (\log^{2 \gamma}(n)) (\uglyfrac{p}{\ep^2})^{1+d/(2\ell)})$ scales better in $n$ than GD's $O(n(p/\ep^2))$, and scales better in $p/\ep^2$ than SGD's $O((p/\ep^2)^2)$. This observation motivates \cref{proposition:p-over-ep-regime,proposition:p-over-epsq-regime}, which show that PPI-GD outperforms the other methods in the important regimes of $p/\ep = O(n^\beta)$ (for strongly convex loss) or $p/\ep^2=O(n^\beta)$ (for non-convex loss) for any $\beta > 0$. While this holds true for LPI-GD as well, our algorithm does not require as strong a bound on $d$, relaxing the curse of dimensionality effect in LPI-GD significantly.

\section{Proofs of Interpolation Error Bounds} \label{section:proofs-of-interpolation-error-bounds}

In this section, we prove \cref{theorem:multivariate-interpolation-bound}, which upper-bounds the uniform norm total interpolation error for $d$-variate functions. As discussed in \cref{subsection:interpolation-error-bounds}, we begin by proving \cref{proposition:multivariate-interpolation-terms}, which decomposes the total error into a sum of iterated interpolation error terms.

\begin{proof}[Proof of \cref{proposition:multivariate-interpolation-terms}]
    We will show by induction that for all $t \in [d]$,
    \begin{equation}
        \norm{\mathrm{P}_{t, \dots, d} g - g}_\infty \leq \sum_{\substack{\mathcal{K} \in \powerset{\bktz{t, d}}: \\ \mathcal{K} \neq \emptyset}} \norm{\mathrm{E}_\mathcal{K} g}_\infty \, . \label{eq:induction-result}
    \end{equation}

    Base case: $t = d$. By inspection, we have
    \begin{equation}
        \norm{\mathrm{P}_d g - g}_\infty \stackclap{(a)}{=} \norm{\mathrm{E}_d g}_\infty = \sum_{\substack{\mathcal{K} \in \powerset{\brc{d}}: \\ \mathcal{K} \neq \emptyset}} \norm{\mathrm{E}_\mathcal{K} g}_\infty \, ,
    \end{equation}
    where (a) holds by the definition of $\mathrm{E}$.
    
    Inductive step: $t < d$. We use ideas from arguments in \cite[Section 4]{CarlsonHall1973} involving \emph{blended interpolants}. We have
    \begin{align}
        &\equad \norm{\mathrm{P}_{t, \dots, d} g - g}_\infty \\
        &\stackclap{(a)}{=} \max_{(x_1, \dots, x_d) \in [a, b]^d} \abs{\mathrm{P}_{t, \dots, d} g(x_1, \dots, x_d) - g(x_1, \dots, x_d)} \\
        &\stackclap{(b)}{=} \max_{(x_1, \dots, x_d) \in [a, b]^d} \, \Biggl| \sum_{i_t=1}^{N+1} \cdots \sum_{i_d=1}^{N+1} \Biggl( \\&\quad g(x_1, \dots, x_{t-1}, u_{i_t}, \dots, u_{i_d}) \prod_{s=t}^d \phi_{i_s}(x_s) \Biggr)- g(x_1, \dots, x_d) \Biggr| \\
        &\stackclap{(c)}{=} \max_{(x_1, \dots, x_d) \in [a, b]^d} \, \Biggl| \sum_{r=t}^d \, \Biggl( \\
        &\quad \sum_{i_r=1}^{N+1} \cdots \sum_{i_d=1}^{N+1} g(x_1, \dots, x_{r-1}, u_{i_r}, \dots, u_{i_d}) \prod_{s=r}^d \phi_{i_s}(x_s) \\
        &\quad - \sum_{\mathclap{i_{r+1}=1}}^{N+1} \cdots \sum_{i_d=1}^{N+1} g(x_1, \dots, x_r, u_{i_{r+1}}, \dots, u_{i_d}) \prod_{\mathclap{s=r+1}}^d \phi_{i_s}(x_s) \Biggr) \Biggr| \, ,
    \end{align}
    where (a) holds by definition of uniform norm and $\max$ may be used due to the continuity of $g$ and the compactness of $[a, b]^d$, (b) holds by definition of $\mathrm{P}$, we assume in (c) and all subsequent steps that the iterated sum $\sum_{i_{r+1}} \cdots \sum_{i_d}$ reduces to a single copy of the summand when $r = d$, and (c) holds because the sum over $r \in \ibktz{t, d}$ telescopes. Proceeding onwards, we have
    \begin{align}
        &\equad \norm{\mathrm{P}_{t, \dots, d} g - g}_\infty \\
        &\stackclap{(d)}{\leq} \max_{(x_1, \dots, x_d) \in [a, b]^d} \sum_{r=t}^d \, \Biggl| \\
        &\quad \sum_{i_r=1}^{N+1} \cdots \sum_{i_d=1}^{N+1} g(x_1, \dots, x_{r-1}, u_{i_r}, \dots, u_{i_d}) \prod_{s=r}^d \phi_{i_s}(x_s) \\
        &\quad - \sum_{i_{r+1}=1}^{N+1} \cdots \sum_{i_d=1}^{N+1} g(x_1, \dots, x_r, u_{i_{r+1}}, \dots, u_{i_d}) \prod_{\mathclap{s=r+1}}^d \phi_{i_s}(x_s) \Biggr| \\
        &= \max_{(x_1, \dots, x_d) \in [a, b]^d} \sum_{r=t}^d \, \Biggl| \sum_{i_{r+1}=1}^{N+1}\! \cdots \sum_{i_d=1}^{N+1} \\
        &\quad \Biggl( \sum_{i_r=1}^{N+1} g(x_1, \dots, x_{r-1}, u_{i_r}, \dots, u_{i_d}) \, \phi_{i_r}(x_r) \\
        &\quad - g(x_1, \dots, x_r, u_{i_{r+1}}, \dots, u_{i_d}) \Biggr) \prod_{s=r+1}^d \phi_{i_s}(x_s) \Biggr| \\
        &\stackclap{(e)}{=} \!\!\!\!\!\!\!\! \max_{(x_1, \dots, x_d) \in [a, b]^d} \!\sum_{r=t}^d  \Biggl| \sum_{i_{r+1}=1}^{N+1}\! \cdots \sum_{i_d=1}^{N+1} \Bigl( \!\mathrm{P}_r g(x_1, \dots, x_r, u_{i_{r+1}}, \dots, u_{i_d}) \\
        &\quad - g(x_1, \dots, x_r, u_{i_{r+1}}, \dots, u_{i_d}) \Bigr) \prod_{s=r+1}^d \phi_{i_s}(x_s) \Biggr| \\
        &\stackclap{(f)}{=} \max_{(x_1, \dots, x_d) \in [a, b]^d} \sum_{r=t}^d \, \Biggl| \sum_{i_{r+1}=1}^{N+1}\! \cdots \sum_{i_d=1}^{N+1} \\
        &\quad \mathrm{E}_r g(x_1, \dots, x_r, u_{i_{r+1}}, \dots, u_{i_d}) \prod_{s=r+1}^d \phi_{i_s}(x_s) \Biggr| \\
        &\stackclap{(g)}{=} \max_{(x_1, \dots, x_d) \in [a, b]^d} \sum_{r=t}^d \abs{\mathrm{P}_{r+1, \dots, d} \mathrm{E}_r g(x_1, \dots, x_d)} \, ,
    \end{align}
    where (d) holds by the triangle inequality, (e) holds by definition of $\mathrm{P}$, (f) holds by definition of $\mathrm{E}$, we assume in (g) and all subsequent steps that $\mathrm{P}_{r+1, \dots, d}$ is the identity operation when $r = d$, and (g) holds by definition of $\mathrm{P}$. We note that steps (c) through (g) correspond to the decomposition of blue nodes in \cref{figure:recursion-tree} and, for example, \cref{eq:proof-sketch-step-a} in the proof sketch in \cref{subsection:interpolation-error-bounds}. Proceeding onwards, we have
    \begin{align}
        &\equad \norm{\mathrm{P}_{t, \dots, d} g - g}_\infty \\
        &\stackclap{(h)}{\leq} \sum_{r=t}^d \max_{(x_1, \dots, x_d) \in [a, b]^d} \abs{\mathrm{P}_{r+1, \dots, d} \mathrm{E}_r g(x_1, \dots, x_d)} \\
        &\stackclap{(i)}{=} \sum_{r=t}^d \norm{\mathrm{P}_{r+1, \dots, d} \mathrm{E}_r g}_\infty \\
        &\stackclap{(j)}{\leq} \sum_{r=t}^d \norm{\mathrm{E}_r g}_\infty + \sum_{r=t}^{d-1} \norm{\mathrm{P}_{r+1, \dots, d} \mathrm{E}_r g - \mathrm{E}_r g}_\infty \\
        &\stackclap{(k)}{\leq} \sum_{r=t}^d \norm{\mathrm{E}_r g}_\infty + \sum_{r=t}^{d-1} \sum_{\substack{\mathcal{K} \in \powerset{\bktz{r+1, d}}: \\ \mathcal{K} \neq \emptyset}} \norm{\mathrm{E}_\mathcal{K} \mathrm{E}_r g}_\infty \\
        &\stackclap{(l)}{=} \sum_{r=t}^d \norm{\mathrm{E}_r g}_\infty + \sum_{r=t}^{d-1} \sum_{\substack{\mathcal{K} \in \powerset{\bktz{r, d}}: \\ r \in \mathcal{K}, \, \abs{\mathcal{K}} \geq 2}} \norm{\mathrm{E}_\mathcal{K} g}_\infty \\
        &\stackclap{(m)}{=} \sum_{r=t}^d \sum_{\substack{\mathcal{K} \in \powerset{\bktz{r, d}}: \\ r \in \mathcal{K}}} \norm{\mathrm{E}_\mathcal{K} g}_\infty \stackclap{(n)}{=} \sum_{\substack{\mathcal{K} \in \powerset{\bktz{t, d}}: \\ \mathcal{K} \neq \emptyset}} \norm{\mathrm{E}_\mathcal{K} g}_\infty \, ,
    \end{align}
    where (h) holds by independently maximizing each term of the sum $\sum_{r=t}^d$, (i) holds by definition of uniform norm, (j) holds by the triangle inequality, (k) holds by the induction hypothesis, (l) holds by definition of $\mathrm{E}$, (m) combines the singleton and non-singleton sequences with lowest term $r$, and (n) combines the sequences in $\powerset{\bktz{t, d}}$ grouped by their lowest term $r$. We note that step (j) corresponds to the decomposition of red nodes in \cref{figure:recursion-tree} and, for example, \cref{eq:proof-sketch-step-b} in the proof sketch in \cref{subsection:interpolation-error-bounds}. Furthermore, each copy of the inner sum in step (m) is illustrated as a subtree in \cref{figure:recursion-tree}.
    
    Lastly, we obtain \cref{proposition:multivariate-interpolation-terms} as a specific case of the inductive argument above. We have
    \begin{equation}
        \norm{\mathrm{P} g - g}_\infty = \norm{\mathrm{P}_{1, \dots, d} g - g}_\infty \stackclap{(a)}{\leq} \sum_{\substack{\mathcal{K} \in \powerset{[d]}: \\ \mathcal{K} \neq \emptyset}} \norm{\mathrm{E}_\mathcal{K} g}_\infty
    \end{equation}
    as desired, where (a) holds by \cref{eq:induction-result}.
\end{proof}

Next, we prove \cref{proposition:univariate-interpolation-bound}, which bounds univariate interpolation errors.

\begin{proof}[Proof of \cref{proposition:univariate-interpolation-bound}]
    Given any $x \in [a, b]$ distinct from the abscissae $\mathcal{G}$, let $\mathrm{Q}_x h: [a, b] \rightarrow \mathbb{R}$ be the degree-$(N + 1)$ polynomial which interpolates $h$ at the points $\mathcal{G} \cup \brc{x}$. We have
    \begin{align}
        \norm{\mathrm{P} h - h}_\infty \!\!&\stackclap{(a)}{=} \max_{x \in [a, b] - \mathcal{G}} \abs{\mathrm{P} h(x) - h(x)} \\
        &\stackclap{(b)}{=} \max_{x \in [a, b] - \mathcal{G}} \abs{\mathrm{Q}_x h(x) - \mathrm{P} h(x)} \\
        &\stackclap{(c)}{=} \max_{x \in [a, b] - \mathcal{G}} \abs{h[u_1, \dots, u_{N+1}, x] \prod_{i=1}^{N+1} (x - u_i)} \\
        &\stackclap{(d)}{\leq} N! \prn{\Delta u}^{N+1} \!\!\!\!\!\! \max_{x \in [a, b] - \mathcal{G}}\! \abs{h[u_1, \dots, u_{N+1}, x]} , \label{eq:interpolation-error-bound}
    \end{align}
    where it suffices to take the maximum in (a) over $x \notin \mathcal{G}$ because $\mathrm{P} h$ interpolates $h$ at all points in $\mathcal{G}$, (b) holds because $\mathrm{Q}_x h(x) = h(x)$ by definition and $\mathrm{Q}_x h$ is well-defined since $x \notin \mathcal{G}$, (c) holds because the Newton basis representation of $\mathrm{Q}_x h$ is (cf. \cite[p. 308]{AscherGreif2011})
    \begin{equation}
        \mathrm{Q}_x h(\tau) = \mathrm{P} h(\tau) + h[u_1, \dots, u_{N+1}, x] \prod_{i=1}^{N+1} (\tau - u_i) \, ,
    \end{equation}
    and (d) holds by the even spacing of the grid $\mathcal{G}$.
    
    Next, we will show by induction that for any $t \geq q$ and any distinct (not necessarily sorted) points $\brc{z_i}_{i=1}^{t+1} \subset [a, b]$,
    \begin{align}
        \abs{h[z_1, \dots, z_{t+1}]} \leq \prn{\frac{2}{b - a}}^{t-q} \frac{\norm{h^{(q)}}_\infty}{q!} \, . \label{eq:divided-difference-bound}
    \end{align}
    
    Base case: $t = q$. By inspection, we have
    \begin{align}
        \abs{h[z_1, \dots, z_{t+1}]} &= \abs{h[z_1, \dots, z_{q+1}]} \\
        &\stackclap{(a)}{\leq} \frac{\norm{h^{(q)}}_\infty}{q!} = \prn{\frac{2}{b - a}}^{t-q} \frac{\norm{h^{(q)}}_\infty}{q!} \, ,
    \end{align}
    where (a) holds by the mean value theorem for divided differences \cite[p. 312]{AscherGreif2011}.
    
    Inductive step: $t > q$. We have
    \begin{align}
        \abs{h[z_1, \dots, z_{t+1}]} &\stackclap{(a)}{=} \abs{\frac{h[z_2, \dots, z_{t+1}] - h[z_1, \dots, z_t]}{z_{t+1} - z_1}} \\
        &\stackclap{(b)}{\leq} \frac{1}{b - a} \abs{h[z_2, \dots, z_{t+1}] - h[z_1, \dots, z_t]} \\
        &\stackclap{(c)}{\leq} \frac{1}{b - a} \Bigl( \abs{h[z_2, \dots, z_{t+1}]} + \abs{h[z_1, \dots, z_t]} \Bigr) \\
        &\stackclap{(d)}{\leq} \frac{2}{b - a} \prn{\frac{2}{b - a}}^{t-1-q} \frac{\norm{h^{(q)}}_\infty}{q!} \\
        &= \prn{\frac{2}{b - a}}^{t-q} \frac{\norm{h^{(q)}}_\infty}{q!}
    \end{align}
    as desired, where (a) holds by the definition of divided differences \cref{eq:divided-difference}, (b) holds because $z_1$ and $z_{t+1}$ are points in $[a, b]$, (c) holds by the triangle inequality, and (d) holds by the induction hypothesis.
    
    Proceeding from \cref{eq:interpolation-error-bound}, we have
    \begin{align}
        \norm{\mathrm{P} h - h}_\infty &\stackclap{(a)}{\leq} N! \prn{\Delta u}^{N+1} \prn{\frac{2}{b - a}}^{N+1-q} \frac{\norm{h^{(q)}}_\infty}{q!} \\
        &\stackclap{(b)}{=} N! \prn{\Delta u}^{N+1} \prn{\frac{2}{N \Delta u}}^{N+1-q} \frac{\norm{h^{(q)}}_\infty}{q!} \\
        &\leq \frac{2^{N+1-q}}{N} \prn{\Delta u}^q \norm{h^{(q)}}_\infty
    \end{align}
    as desired, where (a) holds by \cref{eq:divided-difference-bound} because $N + 1 \geq q$ and (b) holds by the even spacing of the grid $\mathcal{G}$.
\end{proof}

Next, we prove \cref{proposition:multivariate-term-bounds}, which bounds the uniform norm of an iterated interpolation error in terms of the partial derivatives of the function being interpolated.

\begin{proof}[Proof of \cref{proposition:multivariate-term-bounds}]
    We have
    \begin{align}
        &\equad \norm{\mathrm{E}_{j_1, \dots, j_r} g}_\infty \\
        &\stackclap{(a)}{=} \norm{\mathrm{P}_{j_r} \mathrm{E}_{j_1, \dots, j_{r-1}} g - \mathrm{E}_{j_1, \dots, j_{r-1}} g}_\infty \\
        &\stackclap{(b)}{\leq} \alpha_1 \prn{\Delta u} \norm{\frac{\partial \mathrm{E}_{j_1, \dots, j_{r-1}} g}{\partial x_{j_r}}}_\infty \\
        &\stackclap{(c)}{=} \alpha_1 \prn{\Delta u} \norm{\frac{\partial}{\partial x_{j_r}} \mprn{\mathrm{P}_{j_{r-1}} \mathrm{E}_{j_1, \dots, j_{r-2}} g - \mathrm{E}_{j_1, \dots, j_{r-2}} g}}_\infty \\
        &\stackclap{(d)}{=} \alpha_1 \prn{\Delta u} \norm{\frac{\partial \mathrm{P}_{j_{r-1}} \mathrm{E}_{j_1, \dots, j_{r-2}} g}{\partial x_{j_r}} - \frac{\partial \mathrm{E}_{j_1, \dots, j_{r-2}} g}{\partial x_{j_r}}}_\infty \\
        &\stackclap{(e)}{=} \alpha_1 \prn{\Delta u} \norm{\mathrm{P}_{j_{r-1}} \mprn{\frac{\partial \mathrm{E}_{j_1, \dots, j_{r-2}} g}{\partial x_{j_r}}} - \frac{\partial \mathrm{E}_{j_1, \dots, j_{r-2}} g}{\partial x_{j_r}}}_\infty \\
        &\stackclap{(f)}{\leq} \alpha_1^2 \prn{\Delta u}^2 \norm{\frac{\partial^2 \mathrm{E}_{j_1, \dots, j_{r-2}} g}{\partial x_{j_{r-1}} \partial x_{j_r}}}_\infty \, ,
    \end{align}
    where (a) holds by definition of $\mathrm{E}$, (b) holds by \cref{proposition:univariate-interpolation-bound}, (c) holds by definition of $\mathrm{E}$, (d) holds by linearity of differentiation, (e) holds by \cref{lemma:interchanging-derivative-and-interpolation}, and (f) holds by \cref{proposition:univariate-interpolation-bound}. Repeating steps (c) through (f), we have
    \begin{align}
        &\equad \norm{\mathrm{E}_{j_1, \dots, j_r} g}_\infty \\
        &\leq \alpha_1^{r-1} \prn{\Delta u}^{r-1} \norm{\frac{\partial^{r-1} \mathrm{E}_{j_1} g}{\partial x_{j_2} \cdots \partial x_{j_r}}}_\infty \\
        &\stackclap{(g)}{=} \alpha_1^{r-1} \prn{\Delta u}^{r-1} \norm{\frac{\partial^{r-1}}{\partial x_{j_2} \cdots \partial x_{j_r}} \mprn{\mathrm{P}_{j_1} g - g}}_\infty \\
        &\stackclap{(h)}{=} \alpha_1^{r-1} \prn{\Delta u}^{r-1} \norm{\frac{\partial^{r-1} \mathrm{P}_{j_1} g}{\partial x_{j_2} \cdots \partial x_{j_r}} - \frac{\partial^{r-1} g}{\partial x_{j_2} \cdots \partial x_{j_r}}}_\infty \\
        &\stackclap{(i)}{=} \alpha_1^{r-1} \prn{\Delta u}^{r-1} \norm{\mathrm{P}_{j_1} \mprn{\frac{\partial^{r-1} g}{\partial x_{j_2} \cdots \partial x_{j_r}}} - \frac{\partial^{r-1} g}{\partial x_{j_2} \cdots \partial x_{j_r}}}_\infty \\
        &\stackclap{(j)}{\leq} \alpha_{\ell-r+1} \, \alpha_1^{r-1} \prn{\Delta u}^\ell \norm{\frac{\partial^\ell g}{\partial x_{j_1}^{\ell-r+1} \partial x_{j_2} \cdots \partial x_{j_r}}}_\infty
    \end{align}
    as desired, where (g) holds by definition of $\mathrm{E}$, (h) holds by linearity of differentiation, (i) holds by repeated application of \cref{lemma:interchanging-derivative-and-interpolation}, and (j) holds by \cref{proposition:univariate-interpolation-bound} because $N + 1 \geq \ell \geq \ell - r + 1$.
\end{proof}

Next, we prove \cref{lemma:interchanging-derivative-and-interpolation}, which states that interchanging the order of partial differentiation and partial interpolation with respect to distinct variables gives the same result.

\begin{proof}[Proof of \cref{lemma:interchanging-derivative-and-interpolation}]
    We have
    \begin{align}
        &\equad \frac{\partial \mathrm{P}_{j_1, \dots, j_r} g}{\partial x_k}(x_1, \dots, x_d) \\
        &\stackclap{(a)}{=} \frac{\partial}{\partial x_k} \sum_{i_1=1}^{N+1} \cdots \sum_{i_r=1}^{N+1} g(x_1, \dots, u_{i_s}, \dots, x_d) \prod_{s=1}^r \phi_{i_s}(x_{j_s}) \\
        &\stackclap{(b)}{=} \sum_{i_1=1}^{N+1} \cdots \sum_{i_r=1}^{N+1} \frac{\partial g}{\partial x_k}(x_1, \dots, u_{i_s}, \dots, x_d) \prod_{s=1}^r \phi_{i_s}(x_{j_s}) \\
        &\stackclap{(c)}{=} \mathrm{P}_{j_1, \dots, j_r} \mprn{\frac{\partial g}{\partial x_k}}(x_1, \dots, x_d) \, ,
    \end{align}
    where (a) holds by definition of $\mathrm{P}$, (b) holds because $k$ is distinct from $j_1, \dots, j_r$, and (c) holds by definition of $\mathrm{P}$.
\end{proof}

Finally, we prove \cref{theorem:multivariate-interpolation-bound} using the intermediate results established above.

\begin{proof}[Proof of \cref{theorem:multivariate-interpolation-bound}]
    We have
    \begin{align}
        &\norm{\mathrm{P} g - g}_\infty \stackclap{(a)}{\leq} \sum_{\substack{\mathcal{K} \in \powerset{[d]}: \\ \mathcal{K} \neq \emptyset}} \norm{\mathrm{E}_\mathcal{K} g}_\infty \\
        &\stackclap{(b)}{\leq} \!\!\!\!\sum_{\substack{\mathcal{K} \in \powerset{[d]}: \\ \mathcal{K} \neq \emptyset}}\!\!\!\! \biggl(\!\frac{2^{N-\ell+\abs{\mathcal{K}}}}{N}\!\biggr) \!\biggl(\!\frac{2^{N}}{N}\!\biggr)^{\!\abs{\mathcal{K}} - 1}\!\!\!\!\!\!\!\!\!\!\!\!\! \prn{\Delta u}^\ell\! \norm{\frac{\partial^\ell g}{\partial x_{j_1}^{\ell-\abs{\mathcal{K}}+1} \partial x_{j_2} \cdots \partial x_{j_{\abs{\mathcal{K}}}}}}_\infty \\
        &\stackclap{(c)}{\leq} L_h \prn{\Delta u}^\ell \sum_{\substack{\mathcal{K} \in \powerset{[d]}: \\ \mathcal{K} \neq \emptyset}} \prn{\frac{2^{N-\ell+\abs{\mathcal{K}}}}{N}} \prn{\frac{2^{N}}{N}}^{\abs{\mathcal{K}} - 1} \\
        &\stackclap{(d)}{\leq} L_h \prn{\Delta u}^\ell \sum_{\substack{\mathcal{K} \in \powerset{[d]}: \\ \mathcal{K} \neq \emptyset}} \prn{\frac{2^{N}}{N}}^{\abs{\mathcal{K}}} = L_h \prn{\Delta u}^\ell \sum_{r=1}^d \binom{d}{r} \prn{\frac{2^{N}}{N}}^r \\
        &\stackclap{(e)}{\leq} L_h \prn{\Delta u}^\ell \prn{\frac{2^{N}}{N} + 1}^d
    \end{align}
    as desired, where (a) holds by \cref{proposition:multivariate-interpolation-terms}, (b) holds by \cref{proposition:multivariate-term-bounds} because $\ell \geq d$ and $N + 1 \geq \ell$, (c) holds by \cref{eq:holder-to-sobolev}, (d) holds because $\abs{\mathcal{K}} \leq d \leq \ell$, and (e) holds by the binomial theorem.
\end{proof}

\section{Proofs of Oracle Complexity} \label{section:proofs-of-oracle-complexity}

First, we prove \cref{lemma:gradient-error}.

\begin{proof}[Proof of \cref{lemma:gradient-error}]
    From \cref{theorem:multivariate-interpolation-bound}, we have
    \begin{equation}
        \norm{Pg - g}_\infty \leq \holderconst \frac{1}{m^\holderorder} \biggprn{\frac{2^\ell}{\ell}+1}^d.
    \end{equation}
    Bounding the right hand side by $\delta$ implies that $\norm{Pg - g}_\infty \leq \delta$ if
    $m \geq (\frac{\holderconst}{\delta} (\frac{2^{\holderorder}}{\holderorder} + 1)^d)^\uglyfrac{1}{\holderorder}$.
\end{proof}

Next, we prove \cref{theorem:oracle-complexity}.

\begin{proof}[Proof of \cref{theorem:oracle-complexity}] We follow the general argument in \cite{JadbabaieMakur2024}.
Let the condition number $\sigma\!\triangleq\!\lipschitzconst / \mu\!>\! 1$.
From \cite[Proposition 2]{JadbabaieMakur2024}, the required number of iterations $T$ for an inexact gradient descent algorithm to achieve an $\ep$-approximate solution is 
\begin{equation}
    T =\ceil{\prn{\log\prn{\frac{\sigma}{\sigma-1}}}^{-1}\!\!\!\log \prn{\frac{F(\vec{\theta}^{(0)})-F^*+\frac{p}{2\mu}}{\ep}}}  \label{eq:t-value}
\end{equation}
when the gradient estimation error $\delta$ satisfies 
\begin{equation}
    \delta = \prn{1-\frac{1}{\sigma}}^{\!\!\uglyfrac{T}{2}} \!\!\!\!\geq\! \prn{1-\frac{1}{\sigma}}^{\!\!1/2}\!\!\biggl(\frac{F(\vec{\theta}^{(0)})-F^*+\frac{p}{2\mu}}{\ep}\biggr)^{\!\!-1/2}\!\!\!\!\!\!\!\!\!. \label{eq:delta-value}
\end{equation}
Note that the $\ep < (L_l-\mu)p/(2L_l\mu)$ assumption implies that $T\geq 2$. Then, note that
\begin{align}
    & m^d \leq \biggprn{\biggprn{\frac{\holderconst}{\delta} \biggprn{\frac{2^{\holderorder}}{\holderorder} + 1}^{\!\! d}}^{\!\!\uglyfrac{1}{\holderorder}}+1}^{\!\! d} \\
    &\stackclap{(a)}{\leq} \biggl(\looseholderconst^\uglyfrac{1}{\holderorder} 2^d \prn{\frac{\sigma}{\sigma-1}}^{\!\uglyfracprn{1}{2\holderorder}}\!\!\! \biggprn{\frac{F(\vec{\theta}^{(0)})-F^*+\frac{p}{2\mu}}{\ep}}^{\!\!\uglyfrac{1}{(2\holderorder)}}\!\!\!\!\! + 1\biggr)^{\!\!d} \\
    &\leq (1\!+\!\holderconst)^\uglyfrac{d}{\holderorder} 2^{d^2} \!\mprn{\!\frac{\sigma}{2(\lipschitzconst-\mu)}\!}\!\biggl(\!\frac{p\!+\!2\mu(F(\vec{\theta}^{(0)})\!-\!F^*)}{\ep}\!\biggr)^{\!\!\uglyfracprn{d}{2\holderorder}}\!\!\!\!\!\!\!\!\!\!\!,
\end{align}
where (a) holds since $(2^{\holderorder} / {\holderorder} + 1)^{1/\holderorder} \leq 2$ when $\holderorder \geq 2$.
Thus
\begin{align+}
    \Gamma(\operatorname{PPI-GD}) \leq T m^d \leq& (1+\holderconst) 2^{d^2} \Biggl(\frac{\sigma}{(\lipschitzconst-\mu)\log\prn{\frac{\sigma}{\sigma-1}}}\Biggr) \nonumber\\ &\prn{\frac{p+\Delta}{\ep}}^\uglyfracprn{d}{2\holderorder} \log\prn{\frac{p+\Delta}{2\mu\ep}}, \label{eq:intermediate-bound}
\end{align+}
where $\Delta \triangleq 2\mu(F(\vec{\theta}^{(0)})-F^*)\leq L_l \mu \|\vec{\theta}^{(0)}-\vec{\theta}^*\|_2^2\leq c^2 p L_l \mu$.
Suppose $\holderorder \geq d$ and $d \leq \log^{-\uglyfrac{1}{2}} (2) \log^\gamma(n)$ for some $0 < \gamma < 1/2$. Then, we have
\begin{equation}
    2^{d^2}
    \leq \exp \bigl(\log (2) \bigl(\log^{-\uglyfrac{1}{2}} (2) \log^\gamma(n)\bigr)^2\bigr) 
    =\exp \mprn{\log^{2\gamma} (n)}
\end{equation}
for sufficiently large $n$. Thus, we get
\begin{align}
    \Gamma(\operatorname{PPI-GD}) &\leq (1+\holderconst) 2^{d^2} \Biggl(\frac{\sigma}{(\lipschitzconst-\mu)\log\prn{\frac{\sigma}{\sigma-1}}}\Biggr) \\ &\phantom{{}={000000000}} \prn{\frac{p+\Delta}{\ep}}^\uglyfracprn{d}{2\holderorder} \log\prn{\frac{p+\Delta}{2\mu\ep}} \\
    &\leq C_1 \exp\mprn{\log^{2\gamma}(n)} \!\prn{\frac{p}{\ep}}^{\uglyfracprn{d}{2\ell}} \!\!\mprn{\log\mprn{\frac{p}{\ep}}+C_2},
\end{align}
where $C_1$ and $C_2$ are defined in \cref{eq:c12def}. 
\end{proof}

Next, we prove \cref{proposition:p-over-ep-regime}.

\begin{proof}[Proof of \cref{proposition:p-over-ep-regime}]
Observe that
\begin{align}
    \Gamma(\operatorname{PPI-GD}) &\stackclap{(a)}{\leq} C_1 \exp\prn{\log^{2\gamma}(n)} \prn{\frac{p}{\ep}}^{\uglyfracprn{d}{2\ell}} \log\prn{\frac{p}{\ep}} \\
    &\stackclap{(b)}{=} O(\mathsf{subpoly}(n) (n^\beta)^{d/(2\ell)})= \tilde{O}(n^{(d \beta) / (2 \ell)}),
\end{align}
where (a) is from \cref{theorem:oracle-complexity} and (b) follows from the regime $p / \ep = O(n^\beta)$. A similar argument can be used to prove the oracle complexity results for the remaining methods based on known bounds presented in \cref{tab:oracle_complexity_comparison} for GD and SGD, \cite{JohnsonZhang2013} for SVRG, and \cite{BottouCurtis2018} for MBGD (also see \cite{JadbabaieMakur2024}).
\end{proof}

Now, we apply similar techniques to the non-convex setting. First, we prove \cref{theorem:oracle-complexity-nonconvex}.

\begin{proof}[Proof of \cref{theorem:oracle-complexity-nonconvex}]
First, set $\delta = \varepsilon/(2\sqrt{p})$ in \cref{lemma:gradient-error} to get
$
    \inorm{\widehat{\vec{\nabla}} F(\vec{\theta})-\vec{\nabla} F(\vec{\theta})}_2^2 \leq \uglyfrac{\varepsilon^2}{4}.
$
Next, note that Assumption 1 ($\lipschitzconst$-smoothness) implies
\begin{align}
    &\phantom{{}={}} F(\vec{\theta}^{(t+1)}) \\ &\leq F(\vec{\theta}^{(t)})\!+\!\vec{\nabla}F(\vec{\theta}^{(t)})\tsp (\vec{\theta}^{(t+1)}\!-\vec{\theta}^{(t)})\!+\!\frac{\lipschitzconst}{2} \inorm{\vec{\theta}^{(t+1)}\!-\vec{\theta}^{(t)}}_2^2 \\
    &=F(\vec{\theta}^{(t)})-\alpha \inorm{\vec{\nabla}F(\vec{\theta}^{(t)})}_2^2+\frac{\lipschitzconst \alpha^2}{2} \inorm{\widehat{\vec{\nabla}} F(\vec{\theta}^{(t)})}_2^2 \\&\phantom{{}={}}-\alpha \vec{\nabla}F(\vec{\theta}^{(t)})\tsp \left(\widehat{\vec{\nabla}} F(\vec{\theta}^{(t)})-\vec{\nabla}F(\vec{\theta}^{(t)})\right)\\
    &\leq F(\vec{\theta}^{(t)})-\alpha\inorm{\vec{\nabla}F(\vec{\theta}^{(t)})}_2^2\\&\phantom{{}={}}+\frac{\alpha}{2} \left(\inorm{\vec{\nabla}F(\vec{\theta}^{(t)})}_2^2+\inorm{\widehat{\vec{\nabla}} F(\vec{\theta}^{(t)})-\vec{\nabla}F(\vec{\theta}^{(t)})}_2^2\right) \\&\phantom{{}={}}+\lipschitzconst\alpha^2 \left(\inorm{\vec{\nabla}F(\vec{\theta}^{(t)})}_2^2+\inorm{\widehat{\vec{\nabla}} F(\vec{\theta}^{(t)})-\vec{\nabla}F(\vec{\theta}^{(t)})}_2^2\right) \\
    &\leq F(\vec{\theta}^{(t)}) + \left(\lipschitzconst\alpha^2-\frac{\alpha}{2}\right) \inorm{\vec{\nabla}F(\vec{\theta}^{(t)})}_2^2 + \frac{\varepsilon^2}{4} \left(\lipschitzconst\alpha^2+\frac{\alpha}{2}\right) \\
    &= F(\vec{\theta}^{(t)}) - \frac{1}{16\lipschitzconst} \inorm{\vec{\nabla}F(\vec{\theta}^{(t)})}_2^2 + \frac{3\varepsilon^2}{64\lipschitzconst}.
\end{align}
(Recall that $\alpha = 1/(4\lipschitzconst)$ is the step size.) Thus,
\begin{equation}
    \sum_{t=0}^{T-1} \inorm{\vec{\nabla}F(\vec{\theta}^{(t)})}_2^2 \leq 16 \lipschitzconst (F(\vec{\theta}^{(0)})-F^*)+\frac{3}{4}T \varepsilon^2,
\end{equation}
so if $T \geq 64 \lipschitzconst(F(\vec{\theta}^{(0)})-F^*)/\varepsilon^2$,
\begin{equation}
    \frac{1}{T} \sum_{t=0}^{T-1} \inorm{\vec{\nabla}F(\vec{\theta}^{(t)})}_2^2 \leq \frac{16 \lipschitzconst}{T} (F(\vec{\theta}^{(0)})-F^*)+\frac{3}{4} \varepsilon^2 \leq \varepsilon^2.
\end{equation}
This implies that a $\varepsilon$-stationary point (as defined in \cref{eq:ep-stationary}) is reached at least once.
Therefore,
\begin{align}
    &\Gamma(\operatorname{PPI-GD}) \leq T m^d \\
    & \leq \!\biggprn{\!\frac{64 \lipschitzconst(F(\vec{\theta}^{(0)})\!-\!F^*)}{\varepsilon^2}\!+\!1\!} \biggl(\!\biggl(\frac{2\sqrt{p}\holderconst}{\varepsilon} \biggl(\frac{2^{\holderorder}}{\holderorder} + 1\biggr)^{\!\! d}\biggr)^{\!\!\uglyfrac{1}{\holderorder}}\!\!\!+1\biggr)^{\!\! d} \\
    & \leq \!\biggprn{\!\frac{64 \lipschitzconst(F(\vec{\theta}^{(0)})\!-\!F^*)}{\varepsilon^2}\!+\!1\!} \biggprn{\!\biggprn{2^{d+(1/\ell)}\biggprn{\frac{p}{\varepsilon^2}}^{\!\!1/(2\ell)}\!\!\!\!\! L_h^{1/\ell} }\!+\!1}^{\!\! d} \\
    & \leq \!\biggprn{\!\frac{64 \lipschitzconst(F(\vec{\theta}^{(0)})-F^*)}{\varepsilon^2}\!+\!1\!} 2^{d^2+d/\ell}\mprn{\frac{p}{\varepsilon^2}}^{d/(2\ell)}\!\!\! (1+\holderconst)^{d/\ell}  \\
    & \stackclap{(a)}{\leq} C_3 \exp\prn{\log^{2\gamma}(n)} \mprn{\frac{p}{\varepsilon^2}}^{1+d/(2\ell)}
\end{align}
where (a) follows from $2^{d^2} \leq \exp\prn{\log^{2\gamma}(n)}$ and $F(\vec{\theta}^{(0)})-F^*\leq c^2 p L_l/2$ as shown in the proof of \cref{theorem:oracle-complexity}, and $C_3$ is defined in \cref{theorem:oracle-complexity-nonconvex}. %

We note that this flavor of analysis in the non-convex setting is standard in the literature (see, e.g., \cite{LiQian2023,CossonJadbabaie2023b}).
\end{proof}

Finally, we prove \cref{proposition:p-over-epsq-regime}.
\begin{proof}[Proof of \cref{proposition:p-over-epsq-regime}]
    First, \cite[Theorem 2.1]{Vavasis1993} and \cite[Corollary 2.2]{GhadimiLan2013} imply the oracle complexities of GD and SGD to be $O(n(p/\varepsilon^2))$ and $O((p/\varepsilon^2)^2)$, respectively. Plugging in $p/\varepsilon^2=O(n^{\beta})$ produces the desired result. PPI-GD's oracle complexity can be derived similarly from \cref{theorem:oracle-complexity-nonconvex}, noting that $O(\exp(\log^{2\gamma}(n)))$ is subpolynomial when $\gamma < 1/2$.
\end{proof}

\section{Experiments} \label{section:experiments}

\begin{figure*}[t]
    \centering
    \parbox{0.63\textwidth}{
    \begin{subfigure}[t]{0.20\textwidth}
        \includegraphics[width=\linewidth]{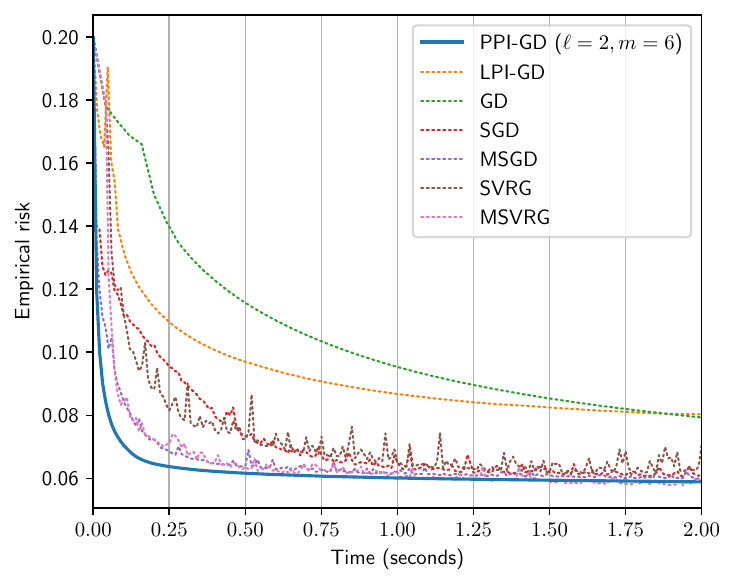}\vspace{-7pt}
        \caption{$h=4,\enspace\nu=0.5$}
        \label{fig:first-left-subfig}
    \end{subfigure}
    \begin{subfigure}[t]{0.20\textwidth}
        \includegraphics[width=\linewidth]{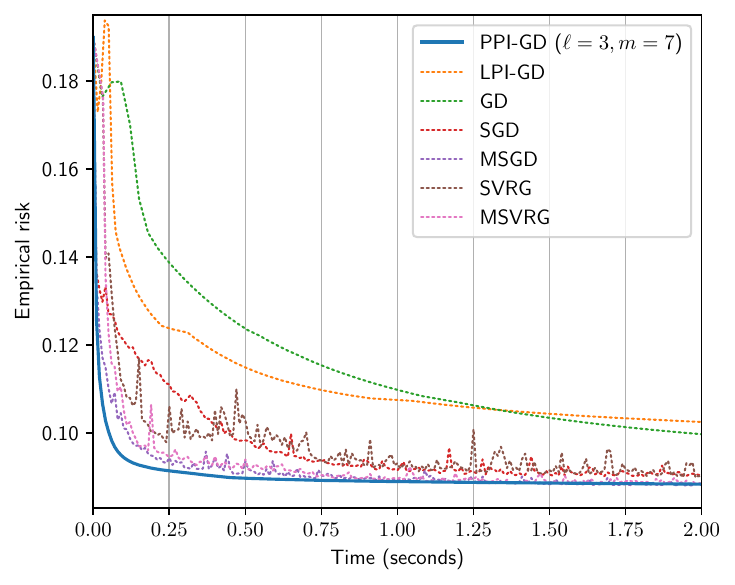}\vspace{-7pt}
        \caption{$h=4,\enspace\nu=1$}
    \end{subfigure}
    \begin{subfigure}[t]{0.20\textwidth}
        \includegraphics[width=\linewidth]{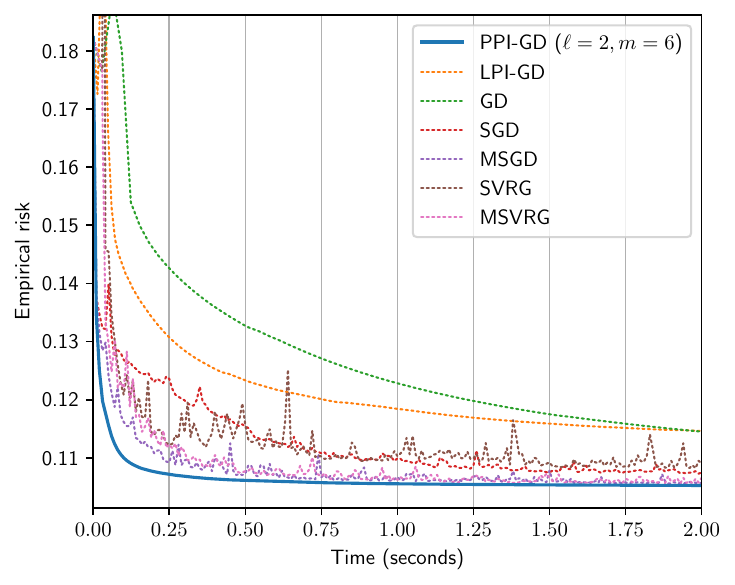}\vspace{-7pt}
        \caption{$h=4,\enspace\nu=1.5$}
        \label{fig:ppi-gd-unique}
    \end{subfigure}
    
    \begin{subfigure}[t]{0.20\textwidth}
        \includegraphics[width=\linewidth]{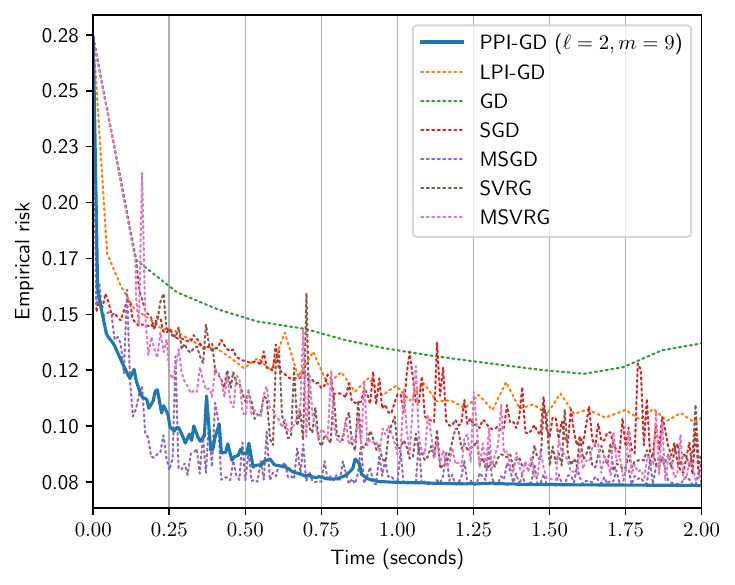}\vspace{-7pt}
        \caption{$h=16,\enspace\nu=0.5$}
    \end{subfigure}
    \begin{subfigure}[t]{0.20\textwidth}
        \includegraphics[width=\linewidth]{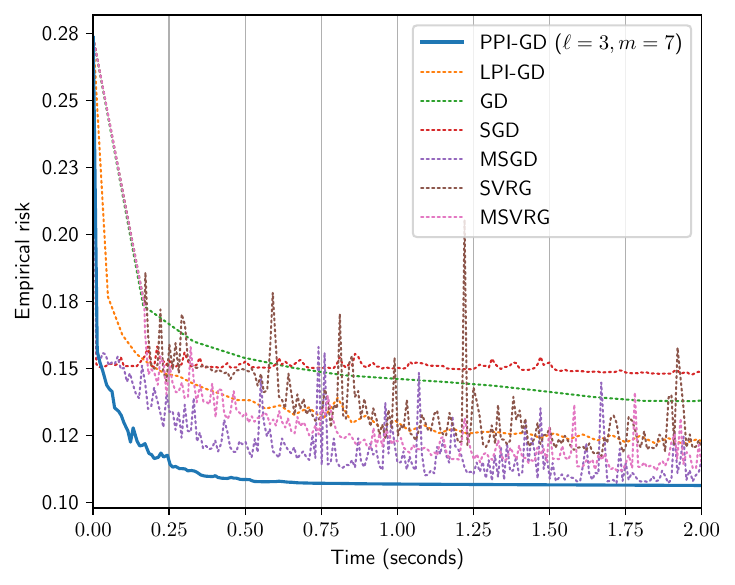}\vspace{-7pt}
        \caption{$h=16,\enspace\nu=1$}
        \label{fig:tradeoff-fig-setting}
    \end{subfigure}
    \begin{subfigure}[t]{0.20\textwidth}
        \includegraphics[width=\linewidth]{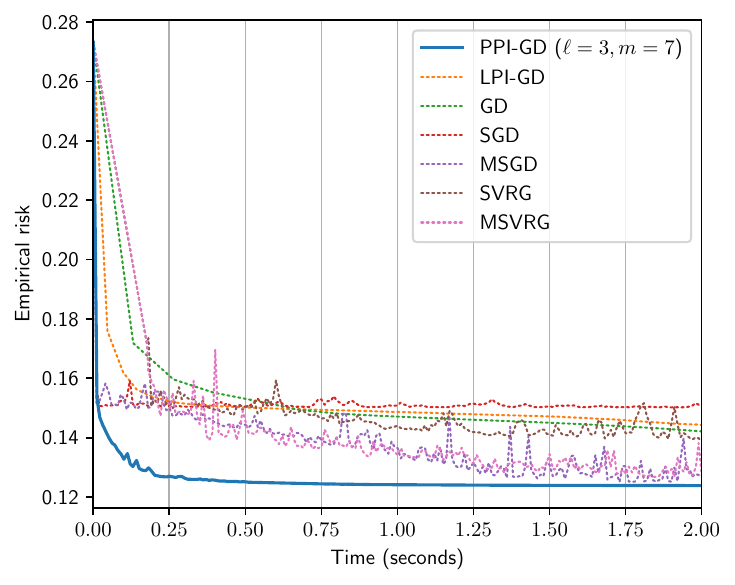}\vspace{-7pt}
        \caption{$h=16,\enspace\nu=1.5$}
        \label{fig:last-left-subfig}
    \end{subfigure}
    
    \begin{subfigure}[t]{0.20\textwidth}
        \includegraphics[width=\linewidth]{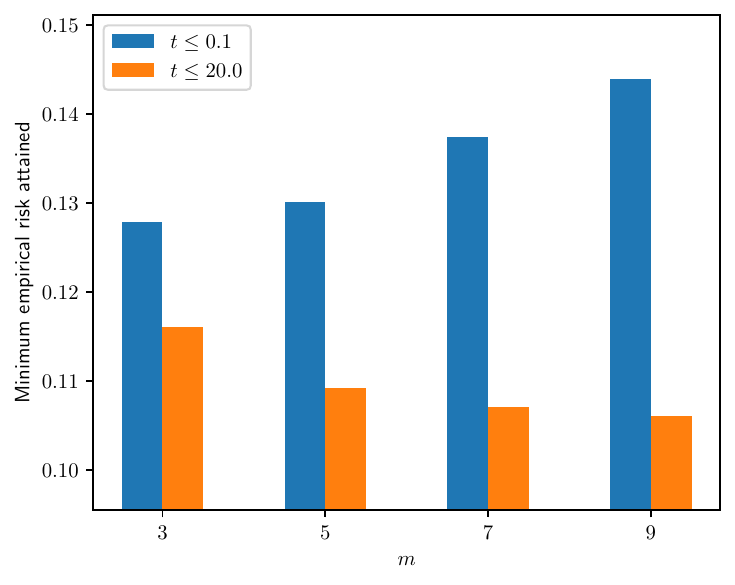}\vspace{-7pt}
        \caption{Speed/bias tradeoff}
        \label{fig:tradeoff-fig}
    \end{subfigure}
    \begin{subfigure}[t]{0.20\textwidth}
        \includegraphics[width=\linewidth]{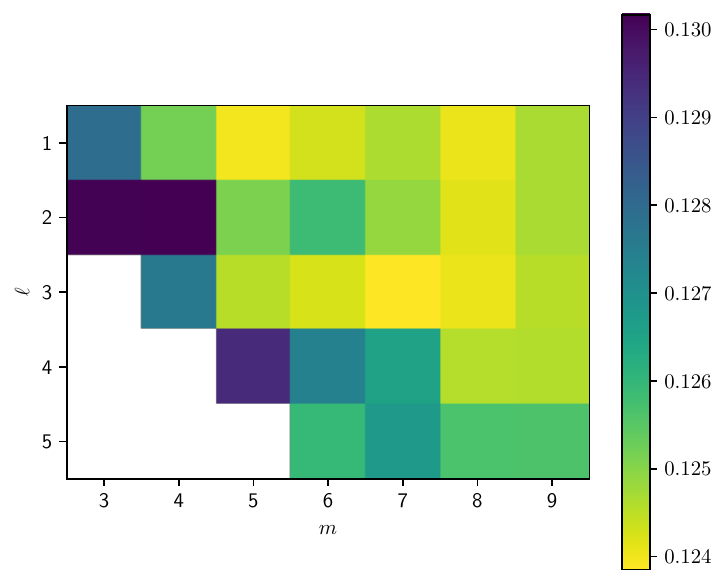}\vspace{-7pt}
        \caption{Hyperparameter search}
        \label{fig:heatmap}
    \end{subfigure}
    \begin{subfigure}[t]{0.20\textwidth}
        \includegraphics[width=\linewidth]{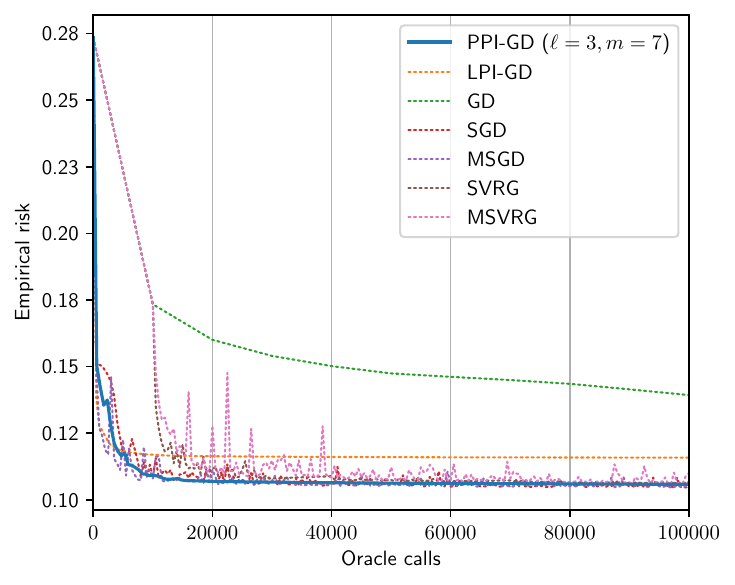}\vspace{-7pt}
        \caption{Oracle complexity}
    \end{subfigure}
    }
    \unskip\ \vrule\ 
    \parbox{0.23\textwidth}{
    \begin{subfigure}[t]{0.20\textwidth}
        \includegraphics[width=\linewidth]{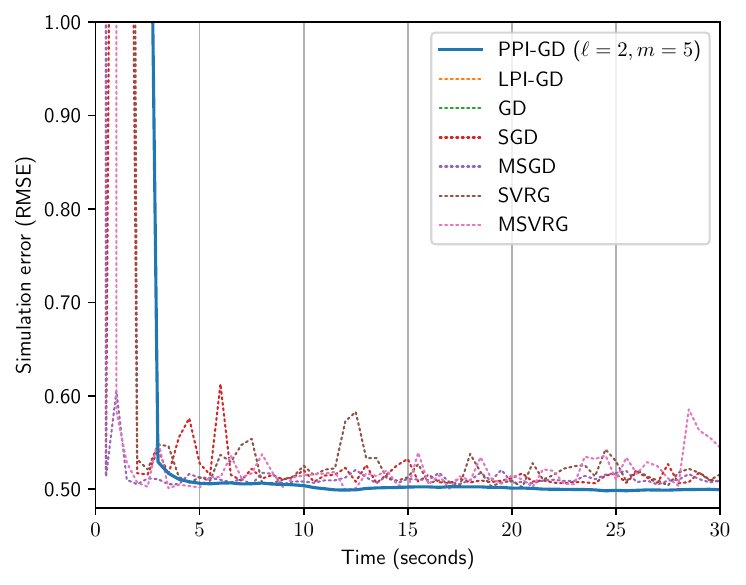}\vspace{-7pt}
        \caption{System identification}
        \label{fig:system-identification}
    \end{subfigure}
    
    \begin{subfigure}[t]{0.20\textwidth}
        \includegraphics[width=\linewidth]{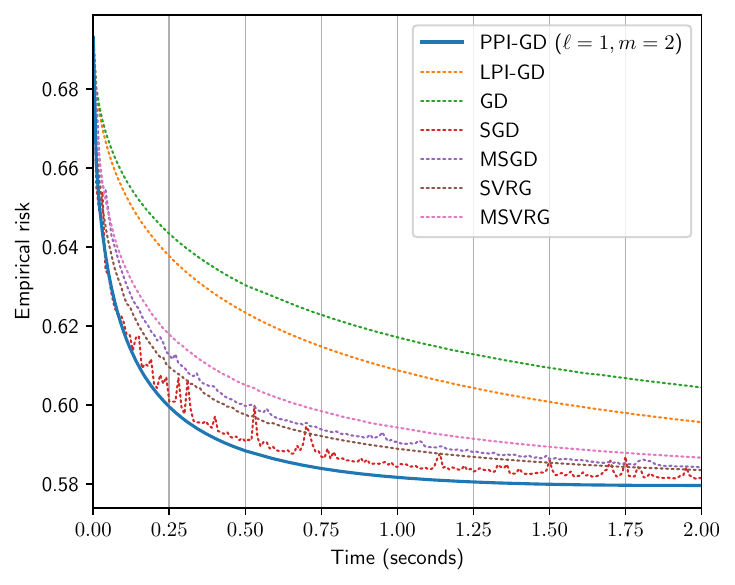}\vspace{-7pt}
        \caption{Logistic regr. ($d=6$)}
        \label{fig:logistic-regression}
    \end{subfigure}
    
    \begin{subfigure}[t]{0.20\textwidth}
        \includegraphics[width=\linewidth]{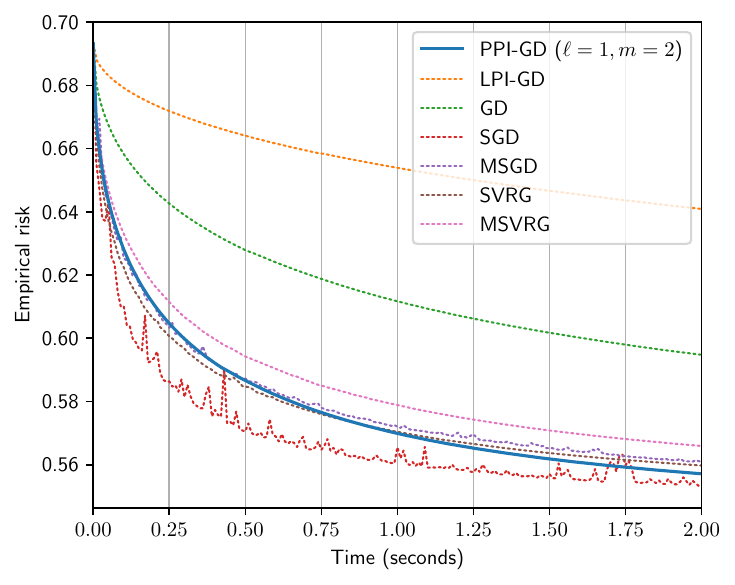}\vspace{-7pt}
        \caption{Logistic regr. ($d=9$)}
        \label{fig:logistic-regression-2}
    \end{subfigure}
    }
    \caption{\textbf{Left}: Visualizations for \cref{subsection:training-nn}. \textbf{(a)--(f)}: Performance comparison of PPI-GD, LPI-GD, GD, SGD, MSGD, SVRG, and MSVRG when used to train a neural network. $h \in \brc{4, 16}$ is the depth of the neural network and $\nu \in \brc{0.5, 1, 1.5}$ is the noise in data. \textbf{(g)}: Comparison between the minimum empirical risk attained in the first $\brc{0.1, 20}$ seconds when PPI-GD with $\ell=1$ and $m\in\brc{3,5,7,9}$ is used to train a neural network, showing a tradeoff between initial rate of convergence and long-term bias. \textbf{(h):} Visualization of the hyperparameter optimization process using grid search for $\ell\in\{1,2,3,4,5\},m\in\{3,4,5,6,7,8,9\}$ (data from experiment (f)). \textbf{(i)}: Learning curve in terms of oracle calls ($h=16, \nu=1.5$). \\
    \textbf{Right}: Visualizations for \cref{subsection:system-id,subsection:logreg}. \textbf{(j)}: Comparison of the simulation performance of a nonlinear neural network ARX system identification model trained with various optimization methods (as a function of training time). \textbf{(k)--(l)}: Learning curves of various gradient descent methods performing logistic regression on datasets with varying dimensionality ($d\in\{6,9\}$).}
    \label{fig:plot-nn} 
\end{figure*}
Experiments were conducted to evaluate the empirical performance of PPI-GD compared to other GD methods in various settings and identify contexts where it performs well. The models and optimization algorithms were implemented with the NumPy Python library. All experiments were run with on a machine with an Intel Core i7-8700K CPU and 64 GB of memory.

\subsection{Training Neural Networks}\label{subsection:training-nn}

First, we compared the effectiveness of PPI-GD and various other optimization algorithms for training neural networks and investigated how performance is affected by model size and noise in data. This experiment also illustrates how grid search methods can be used to determine optimal values for hyperparameters such as $\ell$ and $m$ in practice.
Three synthetic classification datasets were generated for the experiments. Each dataset is parameterized by the variable $\nu\in\brc{0.5, 1, 1.5}$, which captures the ``noisiness'' of the dataset. Each dataset consists of $n=10000$ randomly generated samples in $[0, 1]^2 \times \brc{0, 1}$. Each data point is of the form $(\vec{x}, y)=(x_1, x_2, \iIv{\sin(5x_1)+\sin(5x_2)+\xi \geq \tilde{y}})$, where $x_1, x_2 \sim \operatorname{Unif}(0, 1)$, $\xi \sim \mathcal{N}(0, \nu)$, and $\tilde{y}$ is chosen such that the classes $y=0$ and $y=1$ are balanced, i.e., $\tilde{y}$ is set to be the median of $\{\sin(5x_1)+\sin(5x_2)+\xi\}$.

Fully connected leaky ReLU (rectified linear unit) neural networks were used to classify the datasets. Each neural network has $h \in \brc{4, 16}$ hidden layers with $16$ nodes each. The number of trainable parameters in each model is $p = 272h-207 \in \brc{881, 4145}$.

In each round of simulations, the performance (in terms of wall-clock time) of GD, SGD, MSGD (minibatched SGD), SVRG \cite{JohnsonZhang2013}, MSVRG (minibatched SVRG), LPI-GD \cite{JadbabaieMakur2024}, and PPI-GD was measured. Each algorithm was allowed to run for $2$ seconds, and the training MSE loss was recorded every $0.01$ seconds. For consistency, the same learning rate schedule (inverse square root with an initial value of $0.2$) was used for every optimizer. A batch size of $5$ was used for MSGD and MSVRG. The update frequency of SVRG and MSVRG was set to $n=10000$. The hyperparameters $\ell$ and $m$ for PPI-GD  were chosen through a systematic grid search for all $(\ell, m)\in\{\ell\in\{1,2,3,4,5\},m\in\{3,4,5,6,7,8,9\}: \ell < m\}$ using empirical risk as the metric. The selected hyperparameter values are shown in the legend of each subfigure in \cref{fig:plot-nn}. For LPI-GD, $\ell=1$, $m=5$ were chosen. Note that precomputation steps such as line 6 of \cref{alg:ppigd} (which can be performed in advance independently of the model) were excluded from the wall-clock time measurements. This is necessary to accommodate for LPI-GD, for which precomputation can take several seconds due to matrix inverse operations (on the other hand, the precomputation time of PPI-GD is in the order of milliseconds and mostly negligible).

\Cref{fig:plot-nn}(a)--(i) shows the learning curves for each combination of dataset and model, and the first six columns of \cref{tab:experiments} present the lowest loss attained. 
As expected, the optimizers converge more quickly when the model is simpler (i.e., when $h$ is smaller) or when the data is less noisy (i.e., when $\nu$ is smaller). Notably, the performance of stochastic methods such as SGD degrades quickly in the face of noise. This is because stochastic optimization performs well on datasets that can be easily generalized from a small subset, and noise makes this much more difficult.

Although PPI-GD tends to improve very quickly at first, it soon converges to a suboptimal point. This is because PPI-GD, unlike some other inexact gradient methods like SGD, is biased. As a result, unlike SGD, PPI-GD does not converge to the exact solution even when $T \to \infty$. When the grid size is small, each iteration is very computationally efficient, but the approximate gradient $\widehat{\vec{\nabla}} F(\vec{\theta}^{(t)})$ is less accurate. As $m$ increases, the initial rate of convergence becomes slower, but the algorithm ultimately converges to a more accurate solution.
This tradeoff is explored in \cref{fig:plot-nn}(g), which compares the minimum empirical risk achieved in the first $\brc{0.1,20}$ seconds when PPI-GD with $\ell=1$ and $m\in\brc{3,5,7,9}$ is used in the setting of \cref{fig:plot-nn}(e) (i.e., when training a $16$-layer neural network). To minimize the effect of randomness, $10$ datasets were generated with different random seeds and the results were averaged. The plot shows that using smaller values of $m$ can help attain lower empirical risk in the first $0.1$ seconds (i.e., can converge faster initially), but larger values of $m$ achieve lower empirical risk in the long term (due to having less bias).

\cref{fig:plot-nn}(h), a heat map that visualizes the hyperparameter grid search results for round (f) of the experiment, also provides some interesting insights. Contrary to intuition, values of $m$ that divide $\ell$ evenly (allowing the data space to be cleanly divided into disjoint hypercubes of equal size) are not necessarily optimal. Furthermore, there seem to be multiple local optima, which suggests that gradient-based hyperparameter optimization methods might not be a good choice.

Finally, \cref{fig:plot-nn}(i) shows learning curves in terms of oracle calls (instead of wall-clock time) when $h=16$ and $\nu=1.5$. Although we only include one plot of six due to space constraints, we note that the empirical oracle complexity results corresponding to the other rounds (a)--(e) are mostly similar, with GD and LPI-GD lagging behind PPI-GD and the stochastic methods (SGD, MSGD, etc.), which perform similarly. It might be surprising that the difference in efficiency between PPI-GD and the other methods in terms of oracle calls is not as pronounced as the difference in terms of wall-clock time. This may be due to PPI-GD enjoying better data locality compared to small-batch stochastic methods, but this is beyond the scope of this work.

\begin{table}
    \centering
    \caption{Lowest empirical risk attained in $2$ seconds with each optimizer for each round of the neural network training experiment and the logistic regression experiment, rounded to four decimal places. The best result in each column is highlighted in \textbf{bold}.}
    \label{tab:experiments}
    
    \setlength{\tabcolsep}{1pt} \begin{tabularx}{\columnwidth}{lYYYYYYYY}
    \toprule
    \multirow{2}{*}{Method}& \multicolumn{3}{c}{$4$-layer NN} & \multicolumn{3}{c}{$16$-layer NN} & \multicolumn{2}{c}{Logistic regr.} \\\cmidrule(lr){2-4} \cmidrule(lr){5-7} \cmidrule(lr){8-9}
        & (a) & (b) & (c) & (d) & (e) & (f) & (k) & (l) \\
    \midrule
        PPI-GD & 0.0589 & 0.0884 & \textbf{0.1053} & \textbf{0.0735} & \textbf{0.1063} & \textbf{0.1239} & \textbf{0.5796} & 0.5571 \\
        \midrule
        LPI-GD    & 0.0803 & 0.1025 & 0.1146 & 0.1017 & 0.1216 & 0.1441 & 0.5956 & 0.6409 \\
        GD       & 0.0791 & 0.0997 & 0.1145 & 0.1195 & 0.1378 & 0.1421 & 0.6044 & 0.5948 \\
        SGD      & 0.0607 & 0.0902 & 0.1073 & 0.0783 & 0.1478 & 0.1502 & 0.5813 & \textbf{0.5533} \\
        MSGD     & 0.0598 & \textbf{0.0881} & 0.1056 & 0.0737 & 0.1072 & 0.1251 & 0.5842 & 0.5608 \\
        SVRG     & 0.0617 & 0.0898 & 0.1082 & 0.0799 & 0.1173 & 0.1386 & 0.5835 & 0.5597 \\
        MSVRG    & \textbf{0.0578} & 0.0884 & 0.1056 & 0.0756 & 0.1119 & 0.1267 & 0.5866 & 0.5659 \\
    \bottomrule
    \end{tabularx}
\end{table}

\subsection{System Identification}\label{subsection:system-id}
In addition, we demonstrate an example of a \emph{system identification} task where PPI-GD performs well. The Vibration Test dataset \cite{NoelSchoukens2020} was chosen as the system to be identified, as its nonlinearity and higher order features make it a real-world system that is challenging to identify with conventional techniques. The model is a nonlinear neural network ARX (AutoRegressive with eXogenous input) model with $[u(t), y(t-1), y(t-2)]$ as regressors where $u(t)$ and $y(t)$ are the input and output variables respectively (see \cite{PillonettoAravkinGedon2025} for an overview on the use of neural network methods for system identification). The training dataset consists of $n=73726$ pairs of inputs and outputs. The neural network uses a leaky ReLU activation function and has $8$ fully connected hidden layers with $16$ nodes each. Each optimization algorithm was allowed to run for up to $30$ seconds. While each system identification model was being trained, partially trained models were periodically sampled and used to generate simulated responses given validation input data truncated to $500$ samples. The simulations were compared to the validation output and the RMSE (root mean square error) was recorded, which is visualized in \cref{fig:system-identification}. Note that GD and LPI-GD did not converge to an RMSE of less than $1$ within the time limit, so their results are not visible in the plot. PPI-GD fluctuates less than stochastic methods while providing similar performance.

\subsection{Logistic Regression}\label{subsection:logreg}
Lastly, to find out how the data dimension $d$ affects the performance of PPI-GD compared to other algorithms, we evaluated the empirical performance of each algorithm when used to perform logistic regression on $d$-dimensional data. For each $d\in\{6,9\}$, a dataset consisting of $5000$ randomly generated points in $[0, 1]^{d-1} \times \brc{0, 1}$ was used, where each data point is generated as $(\vec{x}, y)=(x_1,\dots,x_{d-1},\iIv{\sum_{i=1}^{d-1} x_i + \xi \geq \tilde{y}})$ with $x_1,\dots,x_{d-1} \sim \operatorname{Unif}(0, 1)$, $\xi \sim \mathcal{N}(0, 1)$ and $\tilde{y}$ chosen such that there are roughly equal numbers of data points with $y=0$ and $y=1$. Hyperparameters $\ell=1$ and $m=2$ were chosen for PPI-GD (smaller values of $\ell$ and $m$ work better in this setting due to the smoothness of logistic regression compared to a neural network, and also because finer grids scale poorly with $d$). The learning curves for value of $d$ are shown in \cref{fig:plot-nn}(k)--(l), and the lowest empirical risk attained is presented in the last two columns of \cref{tab:experiments}. As expected, PPI-GD loses effectiveness when $d$ increases, with larger values of $m$ being more susceptible to the curse of dimensionality. Still, the high smoothness of the logistic regression loss function allows PPI-GD with a coarser grid to be competitive for moderate values of $d$ such as $d=6$.

\section{Conclusion} \label{section:conclusion}

In this paper, we introduced the PPI-GD algorithm for optimizing ERM objectives which satisfy H\"{o}lder smoothness assumptions in the training data, by using multivariate polynomial interpolation to approximate the true gradient oracle at each iteration. When the data space dimension satisfies $d = O(\log^{0.49}(n))$, we established that the oracle complexity of PPI-GD scales as $\tilde{O}((p/\varepsilon)^{d/(2\ell)})$ in the strongly convex setting and $\tilde{O}((p/\varepsilon^2)^{1+d/(2\ell)})$ in the non-convex setting, showing that our algorithm has lower complexity than GD, SGD, and variants for sufficiently smooth loss functions in the common regime where $p = O(\mathsf{poly}(n))$ and $\varepsilon = \Theta(\mathsf{poly}(1/n))$. PPI-GD improves upon former inexact gradient methods for smooth objectives by relaxing necessary conditions on the scaling of $d$, showing that an astute choice of gradient estimation method contributes substantially towards obtaining algorithms with theoretical gains in oracle complexity in a broad range of domains.

Future research in this vein may incorporate classic techniques such as momentum, Nesterov acceleration, and mini-batching, since the primary innovation of PPI-GD in leveraging smoothness in training data is orthogonal to these techniques in principle. Moreover, our present work builds upon the classical literature on bicubic spline interpolation to derive generalized error bounds on polynomial interpolants, and a fruitful future direction may involve analyzing natural spline interpolants in the same generalized manner. Overall, our main contributions suggest that interesting complexity results remain to be discovered at the intersection of inexact gradient methods and interpolation analysis.

\balance
\bibliographystyle{IEEEtran}   %
\bibliography{arxiv}

\end{document}